\documentclass{article}

\usepackage[utf8]{inputenc}
\usepackage{amsmath,amsthm, amssymb, amsfonts}
\usepackage{todonotes}
\usepackage{kantlipsum}
\usepackage{wrapfig}
\usepackage{paralist}
\usepackage{wasysym}
\usepackage{verbatim}
\usepackage{comment}
\usepackage{fullpage}
\usepackage{mathtools}
\usepackage{graphicx}
\usepackage{wrapfig}
\usepackage{mathrsfs}
\usepackage{enumerate}
\usepackage{dsfont}
\usepackage{units}
\usepackage{ltablex}
\usepackage{varwidth}
\usepackage{pgfplots}
\usepackage{subfigure}
\usepackage{commath}
\usepackage{mathrsfs}
\usepackage{xcolor}
\usepackage{algorithm,algorithmic}
\include{definitions}

\title{VC dimensions of group convolutional neural networks}

\author{Philipp Christian Petersen\thanks{University of Vienna,
    Faculty of Mathematics and Research Network Data Science @ Uni Vienna, Kolingasse 14-16,
    1090 Wien,
    e-mail: \texttt{philipp.petersen@univie.ac.at}
  } \and Anna Sepliarskaia \thanks{e-mail: \texttt{seplanna@gmail.com}}}

\begin{document}

\maketitle

\begin{abstract}
    We study the generalization capacity of group convolutional neural networks. We identify precise estimates for the VC dimensions of simple sets of group convolutional neural networks. In particular, we find that for infinite groups and appropriately chosen convolutional kernels, already two-parameter families of convolutional neural networks have an infinite VC dimension, despite being invariant to the action of an infinite group.
\end{abstract}

\noindent
\textbf{Keywords:} Convolutional neural networks, group convolutional neural networks, sample complexity, generalization, VC dimension

\noindent
\textbf{Mathematics Subject Classification:} 68T07, 68Q32, 68T05

\section{Introduction}

Due to impressive results in image recognition, convolutional neural networks (CNNs) have become one of the most widely-used neural network architectures \cite{krizhevsky2012imagenet, lecun1989backpropagation}.
It is believed that one of the main reasons for the efficiency of CNNs is their ability to convert translation symmetry of the data into a built-in translation-equivariance property of the neural network without exhausting the data to learn the equivariance \cite{bruna2013invariant, mallat2016understanding}.
Based on this intuition, other data symmetries have recently been incorporated into neural network architectures. 
Group convolutional neural networks (G-CNNs) are a natural generalization of CNNs that can be equivariant with respect to rotation \cite{cohen16gcnn,worrall2017harmonic,e2cnn,jenner2022steerable}, scale \cite{sosnovik2020sesn,sosnovik2021disco,bekkers2020bspline}, and other symmetries defined by matrix groups \cite{finzi2021emlp}. 
Moreover, every neural network that is equivariant to the action of a group on its input is a G-CNN, where the convolutions are with respect to the group,  \cite{kondor2018generalization} (see Theorem \ref{thm:equivariantequalsconvolutional} below). 

Although one of the main reasons for constructing equivariant neural networks is their ability to generalize better than neural networks without built-in symmetries, the theoretical understanding of this phenomenon still needs to be better developed. Here, we say that a neural network generalizes if it achieves comparable performance on unseen data compared to the performance on the training data. 

The most studied direction in the analysis of equivariant and invariant models is the analysis of sample complexity.
Several results show that the sample complexity of learning problems is improved if the objective function and the learned algorithm maintain symmetry.
For example, sample complexity is improved by a factor equal to the group size when using an invariant kernel over the group, compared to the corresponding non-invariant kernel \cite{bietti2021sample}.
In \cite{sannai2021improved}, it was observed that a larger volume of a group leads to a smaller generalization error. 
In addition, \cite{elesedy2022group} analyzes sample complexity by calculating a covering number of the input set and concluding that the sample complexity is much smaller for a neural network with invariances because, in general, the covering number of the orbit space representatives is much smaller than the covering number of the original input set.

Overall, the previous work suggests that assuming more invariances of a model implies better generalization behavior. 
\emph{This work can be considered a counter-point to that intuition in the context of G-CNNs.}
Our main contributions are the lower and upper bounds of the VC dimension for a simple two-layer neural network consisting of one convolutional layer and one pooling layer. 

If one fixes the convolutional kernel, then 
the set of associated G-CNNs has only two-free parameters which stem from the involved bias parameters. 
Nonetheless, we found that in the case of an infinite group, for each $n \in \N$, there is a fixed convolution kernel such that the VC dimension of the proposed elementary neural network is at least $n$. 
More precisely, when the symmetry group has size $n\in \N$, then there is a convolutional kernel such that the VC dimension of the proposed neural network is at least $\log_2(n) - 2\log_2(\log_2(n))-4$. The details are given in Theorems~\ref{thm:lowerBound} and Theorem~\ref{thm:lowerBound_general_case} as well as in Corollary~\ref{vs_dim_finite_group}.
The VC dimension of the neural networks with the identified kernel almost matches an associated upper bound, which is $\log_2(n) + 9 \log_2(\log_2(n))$ (see Corollary \ref{cor:lower_bound_VC_dim}).

Our result allows a significant, potentially counter-intuitive conclusion. While a larger group implies more invariances of the associated group-invariant classifiers, it nonetheless yields a higher VC dimension for an appropriately chosen kernel.

This paper is organized as follows: In Section \ref{sec:GCNNs}, we introduce the model for group convolutional neural networks that will be analyzed in the rest of the work. Thereafter, in Section \ref{sec:mainResults}, we present our main results. First, Theorem \ref{thm:upperBound} yields an upper bound for the VC dimension of a G-CNN with a finite group. Second, in Theorem \ref{thm:lowerBound}, we demonstrate an associated lower bound for a carefully chosen kernel $\mathcal{K}$. 

\section{Group convolutional neural networks}\label{sec:GCNNs}
\subsection{Preliminaries}

This paper studies the relationship between the symmetry group and the generalization capabilities of associated group convolutional neural networks.
First, we formally introduce all necessary concepts. We start with the definition of the type of group used in this work. 
This group later encodes the symmetries of the associated group convolutional neural networks. To stress this point, we will henceforth often call it symmetry group.

\begin{definition}[Symmetry group]
A \emph{group} $(G, \cdot)$ is a set with an operation that satisfies the following three properties. First \emph{associativity} holds:
\begin{align*}
    g_1 \cdot (g_2 \cdot g_3) = (g_1 \cdot g_2) \cdot g_3,\ \text{ for all } g_1,g_2,g_3\in G. \end{align*}
Second, there exists an \emph{identity element} $e \in G$ such that
\begin{align*}
    g \cdot e = e \cdot g = g,\ \text{ for all } g\in G.
\end{align*}
Third, for all $g \in G$ there exists a unique \emph{inverse} $g^{-1} \in G$ such that
\begin{align*}
    g^{-1}\cdot g = g \cdot g^{-1} = e.
\end{align*} 
Finally, a group is a \emph{topological group} if it is equipped with a topology such that multiplication with an element and inversion of elements are continuous operations.  We will assume all groups in the sequel to be topological groups. Moreover, we assume the topology to be Hausdorff and first-countable.
\end{definition}

In the sequel, groups act on certain sets, and we expect the group convolutional neural networks to interact appropriately with the corresponding action. To formalize this, we proceed by defining an action. 
\begin{definition}[Action of a symmetry group]
Let $G$ be a symmetry group. An \emph{action} of $G$ on a set $\mathcal{X}$ is a map $a: G \times \mathcal{X} \rightarrow \mathcal{X}$ such that for all $g_1, g_2\in G$ and for all $x\in \mathcal{X}$
\begin{align*}
    a(g_1 \cdot g_2, x) = a(g_1, a(g_2, x)).   
\end{align*}
\end{definition}

We are interested in neural networks that are invariant to specific group actions. 
This means that the output of a neural network does not change if a group action is applied to the input. 
The standard architecture of an invariant neural network consists of two parts: First, a so-called equivariant neural network, and second, a pooling operation. 
To clarify this concept, we first recall equivariant maps.
\begin{definition}[Equivariant map]\label{def:equivariantMap}
Let $\mathcal{X}, \mathcal{Y}$ be sets, let $G$ be a group, and $a_X$, $a_Y$ be group actions on $\mathcal{X}$ and $\mathcal{Y}$, respectively. 
A map $f:\mathcal{X} \rightarrow \mathcal{Y}$ is \emph{equivariant} to the actions $a_X, a_Y$ if for all $g\in G, x\in\mathcal{X}$ it holds that
\begin{align*}
    f(a_X(g, x)) = a_Y(g, f(x)).
\end{align*}
\end{definition}
Definition \ref{def:equivariantMap} requires that the group transformations commute with the application of $f$. In other words, transforming the input before the application of $f$ is equivalent to transforming the output after the application of $f$. 

Equivariant neural networks define parametric families of equivariant maps by composing layers that are individually equivariant with respect to the same group.
\begin{definition}[Equivariant neural network~\cite{lim2022equivariant}]\label{def:equiNN}
Let $n \in \N$, and $\mathcal{X}_1, \dots, \mathcal{X}_n$ be sets. Let $G$ be a group, and let $a_j:G \times \mathcal{X}_j \rightarrow \mathcal{X}_j$ for $j = 1, \dots, n$ be actions of $G$. 
An \emph{equivariant neural network} with respect to the actions $(a_j)_{j=1}^n$ is a function $f: \mathcal{X}_1 \rightarrow \mathcal{X}_n$
which can be described as the composition of linear maps
$f_i:\mathcal{X}_i\rightarrow \mathcal{X}_{i+1}$ that are equivariant with respect to the actions $a_i, a_{i+1}$, and coordinate-wise nonlinearity $\rho_i\colon \R \to \R,\ \rho_i(x) = \rho(x+b_i)$, for $b_i \in \R$:
\begin{align*}
    f(x) = \rho_n\circ f_n\circ \rho_{n-1}\circ f_{n-1}\cdots\rho_1\circ f_{1}(x), \text{ for } x \in \mathcal{X}_1.
\end{align*}
In this definition, for $i \in [n]$, the function $f_i$ is called the \emph{$i$-th layer of the neural network $f$}.
\end{definition}
For comparison, we also recall the classical notion of feed-forward neural networks. 
\begin{definition}[Feed-forward neural network]
Let $n \in \N$ and $\mathcal{X}_1, \dots, \mathcal{X}_n$ be sets.
A \emph{feed-forward neural network} is a function $f: \mathcal{X}_1 \rightarrow \mathcal{X}_n$
which can be described as the composition of affine maps $f_i:\mathcal{X}_i\rightarrow \mathcal{X}_{i+1}$ and coordinate-wise nonlinearity $\rho\colon \R \to \R$:
\begin{align*}
    f(x) = \rho\circ f_n\circ \rho\circ f_{n-1}\cdots\rho\circ f_{1}(x), \text{ for } x \in \mathcal{X}_1.
\end{align*}
In this definition, for $i \in [n]$, the function $f_i$ is called the \emph{$i$-th layer of the neural network $f$}.
\end{definition}
Equivalently to Definition \ref{def:equiNN}, we can define an equivariant neural network as a feed-forward neural network in which each layer is an equivariant function.
Moreover, it is convenient to describe equivariant neural networks through generalized convolutions.
To this end, we introduce the Haar measure on a compact group.

\begin{theorem}[{\cite[Theorem 8.1.2]{procesi2007lie}}]
Let $G$ be a compact group. Then, there is a left-invariant finite Borel measure, i.e., a finite  Borel measure $\mu$, such that $\mu(A) = \mu(g\cdot A)$ for any measurable set $A\subset G$. This measure is called a Haar measure. The Haar measure is unique up to scaling. 
\end{theorem}

\begin{remark}
Every finite topological group is compact. Hence, there exists a Haar measure on every finite group. A Haar measure on the finite group is its counting measure. Since scaling does not affect the results in the rest of the manuscript, we will always choose the counting measure as the Haar measure on a finite group.    
\end{remark}

On compact groups, we can define the generalized convolution. 
We will typically use the Haar measure as a measure in the convolution. However, in a couple of special cases later, we require a bit more generality, which is why we state the definition for general finite measures. 
\begin{definition}[Generalized convolution]
Let $G$ be a compact group, let $\mu$ be a finite measure on $G$, and let $f_1, f_2 \colon G\rightarrow\mathbb{R}$ be functions.

Then, the convolution of $f_1$ with $f_2$ is defined by
\begin{align*}
    (f_1 *_G f_2)(g) = \int_{G}f_1(g\cdot h^{-1})\cdot f_2(h) d\mu(h) \quad \text{ for } g \in G. 
\end{align*}
\end{definition}
\begin{remark}
For the generalized convolution to be sensible, we need for every $g \in G$, the integral $\int_{G}f_1(g\cdot h^{-1})\cdot f_2(h) d\mu(h)$ to be well defined. This is guaranteed, for example, if $f_1,f_2$ are both $\mu$-measurable and bounded, which we will assume in the sequel.
\end{remark}

\begin{theorem}[\cite{kondor2018generalization}]\label{thm:equivariantequalsconvolutional}
A feed-forward neural network $f$ is equivariant to the action of a compact group $G$ on its inputs if and only if each layer of $f$ implements a generalized form of convolution with respect to $G$.
\end{theorem}

Thanks to Theorem \ref{thm:equivariantequalsconvolutional}, each equivariant neural network is based on repeated generalized convolutions and hence can also be called \emph{Group Convolutional Neural Network (G-CNN)}. 
The standard way of designing an invariant classifier or regressor corresponding to an action $a: G\times \mathcal{X} \rightarrow \mathcal{X}$ using neural networks is to compose an equivariant neural network with a global pooling operation~\cite{kondor2018generalization, keriven2019universal, bekkers2018roto}. 
Concretely, if $f:\mathcal{X} \rightarrow \R$ is an equivariant neural network, then we define $\widetilde{f}:\mathcal{X} \rightarrow\mathbb{R} $ as
\begin{align*}
    \widetilde{f}(x) = \int_G f(a(g,x)) d\mu(g), \text{ for } x \in X.
\end{align*}

To illustrate the concepts described above we give an example. 

\begin{example}\label{ex:homspace}
Let $\mathcal{X} \coloneqq \mathbb{Z}^2$, which is a homogeneous space under the standard action of the group of integer translations:
\begin{align*}
T((t_1,t_2),(x_1, x_2)) = (x_1 +t_1, x_2 +t_2).    
\end{align*}

An image, such as an image of a handwritten digit, can be thought of as a function on $\mathcal{X}$ that, given coordinates, returns the pixel value at the corresponding location: $\mathcal{I}:\mathcal{X}\rightarrow\mathbb{R}$.
The standard action of the group of integer translations can be extended to the action on images by
\begin{align*}
T((t_1,t_2), \mathcal{I})(x_1, x_2) = \mathcal{I}(x_1 -t_1, x_2- t_2).    
\end{align*}

The generalized convolution, in this case, is a standard convolution layer, and the classifier, which is invariant to the action of $T$, can be designed by first having $n$ convolution layers and then composing it with an averaging pooling operator.

Similar examples can be constructed with respect to group actions 
 of \emph{compact groups} incorporating rotations, shearings, or general affine transformations. 
\end{example}

\subsection{Introduction of the G-CNN model}
In this paper, we consider the situation already encountered in Example \ref{ex:homspace}, where the input space for a neural network is a subset of functions on a homogeneous space $\mathcal{X}$. 
As mentioned in~\cite{kondor2018generalization} in this case, after fixing an origin $y \in \mathcal{X}$, we can map a
function $f: \mathcal{X} \to \R$ to a function $\bar{f} \colon G \to \R$, by 
\begin{align}\label{eq:defOfFbar}
    \bar{f}(g) &\coloneqq f(a(g,y)) \text{ for } g \in G.
\end{align}
Now $\bar{f}$ can be convolved with a kernel $\mathcal{K}: G \to \R$ if both $\bar{f}$ and $\mathcal{K}$ are measurable and bounded. Clearly, $\bar{f}$ is bounded if $f$ was. Hence, for convenience, we define 
$$
    \mathcal{B}(\mathcal{X}, \mu) \coloneqq \{ f \colon \mathcal{X} \to \R \colon f \text{ bounded and }\bar{f} \text{ is }\mu-\text{measurable} \}, 
$$
Now the generalized convolution between a function $f \in \mathcal{B}(\mathcal{X}, \mu)$ and a kernel $\mathcal{K}: G \to \R$ can be defined as 
\begin{align}
    \label{eq:genConvolution}
    f *_G \mathcal{K} \coloneqq \bar{f} *_G \mathcal{K}.
\end{align}

In this work, we study the generalization capabilities of invariant classifiers with architecture described in the previous subsection. 
These classifiers are a composition of a G-CNN and a pooling operation, followed by the application of a sign function before the output.
Typically, the hypothesis class of this type of classifier consists of a neural network with several layers, where each layer is a generalized convolution with a kernel taken from a linear space with a predetermined basis.
The learning procedure selects suitable kernels by finding their coordinates in a given basis, i.e., the coordinates are the learnable parameters of the training algorithm.
In our analysis, we consider neural networks that have the fewest training parameters, i.e., neural networks with only one convolutional layer and with a fixed kernel.
We state the corresponding definition below.

\begin{definition}[G-CNNs with fixed kernel] \label{def:oneFilterGCNN}
Let $G$ be a compact group acting on $\mathcal{X}$. Let $\mathcal{K} \colon G \to \R$ be a bounded kernel. Let $\mu$ be the Haar measure on $G$. Let $c_1,c_2\in\mathbb{R}$ and set
\begin{align}
   H_{c_1,c_2}(\mathcal{K})\colon \mathcal{B}(\mathcal{X}, \mu) &\to \{-1,1\}\nonumber\\
   H_{c_1,c_2}(\mathcal{K})(f) &= \mathrm{sign}\left(\int_G \relu((f*_G\mathcal{K})(g) + c_1) d\mu(g) + c_2\right).\label{eq:defOfCNN}
\end{align}
Here $\mathrm{sign} = 2 \mathds{1}_{(0,\infty)} - 1$. 
We denote the \emph{set of G-CNNs with kernel $\mathcal{K}$} by $\mathcal{H}(\mathcal{K}) \coloneqq \{ H_{c_1,c_2}(\mathcal{K}) \colon c_1,c_2 \in \R\}$.
\end{definition}

We note that the set $\mathcal{H}(\mathcal{K})$ has only two scalar parameters since the kernel $\mathcal{K}$ is fixed for the whole class.
Even though this neural network contains only two scalar parameters and, in addition, is constrained to be invariant, we will demonstrate in the next section that when the group $G$ contains an infinite number of elements, there exists, for each $n\in \N$, a convolution kernel $\mathcal{K}$ such that the VC dimension of $\mathcal{H}(\mathcal{K})$ is at least $n$. 
By the fundamental theorem of learning (\cite[Theorem 3.20]{mohri2018foundations} or \cite[Theorem 6.7]{shalev2014understanding}), this implies that successful learning of these neural networks from finitely many samples is impossible in general. 
On the other hand, we will see that the VC dimension can be upper-bounded for finite groups.

\section{VC dimension of a G-CNN associated with a finite group}
\label{sec:mainResults}

To study the generalization capacity of sets of G-CNNs of the form $\mathcal{H}(\mathcal{K})$, we compute the so-called \emph{VC dimension} of these sets, which we denote by $\vcdim(\mathcal{H}(\mathcal{K}))$. We refer to \cite[Definition 6.5]{shalev2014understanding} for a formal definition of the VC dimension as well as for the concept of \emph{shattering}. We recall that, by the fundamental theorem of learning \cite[Theorem 6.7]{shalev2014understanding}, a finite VC dimension facilitates learning, whereas an infinite VC dimension prohibits it.

\subsection{Upper bound on VC dimension}
For finite groups, we have the following upper bound on the VC dimension of $\mathcal{H}(\mathcal{K})$.

\begin{theorem}[Upper bound on VC dimension] \label{thm:upperBound}
Let $G$ be a group acting on $\mathcal{X}$, let $|G| = n$, and let $\mu$ be the Haar measure on $G$. Let $\mathcal{K} \colon G \to \R$ be a bounded kernel. 
Then, $\vcdim(\mathcal{H}(\mathcal{K})) = m < \infty$.
Moreover, 
\begin{align}
    m - 3 \log_2(m+1) \leq \log_2(n).   
\end{align}
\end{theorem}
\begin{proof}
%By \eqref{eq:genConvolution}, every generalized convolution of a function $f: \mathcal{X} \to \R$ and the kernel $\mathcal{K}: G \to \R$ is a function on the group $G$. As a consequence, the image of $f *_G \mathcal{K}$ contains no more than $|G| = n$ elements. 

For $f \in \mathcal{B}(\mathcal{X}, \mu)$, we consider the function $F_{\mathcal{K},f}$ defined by
\begin{align}
\label{F_K_f}
    F_{\mathcal{K},f} \colon \R &\to \R\\
     F_{\mathcal{K},f}(c)&\coloneqq \int_{G} (f*_G\mathcal{K})(g) \mathds{1}_{(f*_G\mathcal{K})(g) > -c} d \mu(g). \nonumber
\end{align}

Let $\{g_1, \dots, g_n\} = G$ be such that $\kappa_j \coloneqq (f*_G\mathcal{K})(g_j) \mu(g_{j}) \geq (f*_G\mathcal{K})(g_{j+1}) \mu(g_{j+1}) \eqqcolon \kappa_{j+1}$ for all $j < n$. Then, the function $F_{\mathcal{K},f}$ is a step function with values 
\begin{align}\label{eq:different_values}
    0, \kappa_1, \kappa_1 + \kappa_2,  \dots , \sum_{j= 1}^{n-1} \kappa_j, \sum_{j= 1}^n \kappa_j.
\end{align}
In particular, $|F_{\mathcal{K},f}(\R)| \leq n+1$, where we denote by $|A|$ the cardinality of a finite set $A$.

Assume that, for $m \in \N$, $\vcdim(\mathcal{H}(\mathcal{K})) \geq m$, then there are $m$ functions $f_1,\dots, f_m$ which are shattered by $\mathcal{H}(\mathcal{K})$.
We observe with the help of \eqref{eq:different_values} that $F_{\mathcal{K},f_k}$ has at most $n$ break points---the points where the functions are not affine---for each $f_k$, $k \in [m]$. This yields that the $f_k$ collectively have not more than $mn$ break points. Hence, $F_{\mathcal{K},f_1},\dots,F_{\mathcal{K},f_m}$ have at most $mn+1$ constant pieces.
Consequently, the constant regions of $F_{\mathcal{K},f_1},\dots,F_{\mathcal{K},f_m}$ divide the real line into no more than $m  n + 1$ segments, i.e., there exist $(\Lambda_i)_{i=1}^{mn+1}$ such that $\Lambda_i$ is an interval and $c \mapsto (F_{\mathcal{K},f_1}(c),\dots,F_{\mathcal{K},f_m}(c))$ is constant on $\Lambda_i$ for all $i \in [mn+1]$. 

For $i \in [mn+1]$ and $k \in [m]$, we define maps
\begin{align}\label{eq:theGammas}
    \gamma_{i,k}:\Lambda_i&\rightarrow\mathbb{R},\\
    \gamma_{i,k}(c) &\coloneqq \int_{G} \relu((f_k*_G\mathcal{K})(g) + c) d\mu(g)\nonumber\\
    &= F_{\mathcal{K},f_k}(c) +  \int_{G}  c\cdot\mathds{1}_{(f_k*_G\mathcal{K})(g) > -c} d \mu(g).\nonumber
\end{align}
Note that $\gamma_{i,k}$ is an affine function as $F_{\mathcal{K},f_k}$ is constant on $\Lambda_i$. We proceed by estimating the number of possible classifications of $f_1,\dots, f_m$ by $H_{c_1,c_2}(\mathcal{K})$ if $(c_1,c_2)$ can be chosen in $\Lambda_i \times \R$.

\begin{lemma}\label{lem:LemmaOnLambdai}
Let $i \in [nm + 1]$. The classifications of $f_1,\dots,f_m$ by $H_{c_1,c_2}(\mathcal{K})$ with $c_1 \in \Lambda_i$ and $c_2 \in \R$ correspond to the vectors which are produced by the map:
\begin{align}
    \Lambda_i \times \R \ni (c_1,c_2) \mapsto \mathcal{H}_{c_1, c_2}(f_1,\dots,f_m) \coloneqq \left(  H_{c_1,c_2}(\mathcal{K})(f_1),\dots, H_{c_1,c_2}(\mathcal{K})(f_m)\right).
\end{align}
Then, the number of elements of $\{\mathcal{H}_{c_1, c_2}(f_1,\dots,f_m),\ c_1\in\Lambda_i,\ c_2\in\mathbb{R}\}$ is not more than ${m  (m-1)}/{2} + m$.
\end{lemma}
\begin{proof}
Note that, per Definition~\ref{def:oneFilterGCNN} and \eqref{eq:theGammas}
$$
 H_{c_1,c_2}(\mathcal{K})(f_k) = \left \{ \begin{array}{rl}
    1  & \text{ if }  \gamma_{i,k}(c_1) > - c_2\\
    -1  & \text{ else.}\\
 \end{array}\right.
$$
To improve readability in the remainder of the proof, we omit the explicit reference to the kernel $\mathcal{K}$ and write $H_{c_1,c_2}$ instead of $H_{c_1,c_2}(\mathcal{K})$.

Let $l \in \N$, and $(p_j)_{j=1}^l \subset \Lambda_i$ be the set of intersection points of $(\gamma_{i,k})_{k=1}^m$. More precisely,
\begin{align*}
 (p_j)_{j=1}^l \coloneqq \{p \in \Lambda_i \colon \exists r,s\in [m] \colon \gamma_{i,r}(p) = \gamma_{i,s}(p) \text{ and } \gamma_{i,r} \neq \gamma_{i,s} \text{ on }\Lambda_i\}.
\end{align*}
We assume that $p_1 < p_2 < \cdots < p_l$. For $j \leq l$, we denote the set of all classifications of $f_1,\dots,f_m$ when $c_1\in\Lambda_i,\ c_1 < p_{j},\ c_2\in\mathbb{R}$ by
\begin{align*}
    \mathcal{H}_j \coloneqq \{({H}_{c_1,c_2}(f_1),\dots,{H}_{c_1,c_2}(f_m))\ c_1\in\Lambda_i,\ c_1 < p_{j},\ c_2\in\mathbb{R}\}.
\end{align*} 
We also define 
$$
    \mathcal{H}_{l+1} \coloneqq \{({H}_{c_1,c_2}(f_1),\dots,{H}_{c_1,c_2}(f_m)),\ c_1\in\Lambda_i,\ c_2\in \R\}.
$$
We denote by $\mathcal{P}(A)$ the power set of a set $A$.
Let for $k \in [m]$, $\widetilde{\gamma}_{i,k}: \R \to \R$ be an affine linear function that coincides with $\gamma_{i,k}$ on $\Lambda_i$.
We define $u \colon [m] \times \R \to \mathcal{P}([m])$ by 
$$
    u(k, c) \coloneqq \{ j \in [m] \colon \widetilde{\gamma}_{i,j}(c) > \widetilde{\gamma}_{i,k}(c) \}.
$$

We also set $p_0$ as the smallest intersection point of $(\widetilde{\gamma}_{i,k})_{k=1}^m$ which is smaller than $p_1$ or $-\infty$ if such a point does not exist. Similarly, $p_{l+1}$ is the largest intersection point of $(\widetilde{\gamma}_{i,k})_{k=1}^m$ which is larger than $p_l$ or $\infty$ if such a point does not exist.

Note that $u(k, c)$ is constant on $(p_j, p_{j+1})$ for all $j = 0, \dots, l$. 
Hence, for $j = 1, \dots, l$ and an arbitrary $c^* \in (p_{j}, p_{j+1}) \cap \Lambda_i$ it holds that 
\begin{align}
    \mathcal{H}_{j+1} \setminus \mathcal{H}_{j} &= \{(\mathcal{H}_{c_1, c_2}(f_1),\dots,\mathcal{H}_{c_1, c_2}(f_m)),\ c_1 \in [p_{j}, p_{j+1}) ,\ c_2 \in \R\} \nonumber\\
    &\subset \{ 2 \mathds{1}_{u(k, c^*)} - 1 \colon k \in [m]\} \cup \{ 2 \mathds{1}_{u(k, p_{j})} - 1 \colon k \in [m]\}\nonumber\\
    &\eqqcolon J_{j+1} \cup \{ 2 \mathds{1}_{u(k, p_{j})} - 1 \colon k \in [m]\}, \label{eq:fromCurlytoStraightH}
\end{align}
where $\mathds{1}_{u(k, c^*)}$ is a vector which equals $1$ in coordinate $q$ if $q \in u(k, c^*)$ and $0$ else. The last inclusion holds because only the order of the $(\gamma_{i,k}(c_1))_{k=1}^m$ influences the classifications that can be produced with all $c_2$ and that order is, as explained before, the same for all $c_1 \in (p_{j}, p_{j+1})$.

We set for $c^* \in (p_0, p_1)$
\begin{align*}
    J_{1} \coloneqq \{ 2 \mathds{1}_{u(k, c^*)} - 1 \colon k \in [m]\} \supset \mathcal{H}_1. 
\end{align*}
Note that, by construction for $j = 1, \dots, l$
\begin{align}
\label{eq:borderCase}
    \{ 2 \mathds{1}_{u(k, p_{j})} - 1 \colon k \in [m]\} \subset J_{j} \cap J_{j+1}.
\end{align}
Due to \eqref{eq:fromCurlytoStraightH} and \eqref{eq:borderCase}, we have that 
\begin{align*}
    \bigcup_{j=2}^{l+1} (\mathcal{H}_j \setminus \mathcal{H}_{j-1}) \subset \bigcup_{j=2}^{l+1} J_j.
\end{align*}
Moreover, we set $J_1 \coloneqq \mathcal{H}_1$.

Let $k \in [m]$ be such that $\gamma_k$ intersects no $\gamma_q$ with $q\neq k$ on $(p_{j-1}, p_{j+1})$, i.e.,
$$
    \{q \in [m] \colon \gamma_q(p_j) \geq \gamma_k(p_j)\} = \{q \in [m]\colon \gamma_q \geq \gamma_k \text{ on } (p_{j-1}, p_{j+1})\}.
$$
Then, it is clear that $u(k, c)$ is constant on $(p_{j-1}, p_{j+1})$. Therefore, $J_j \setminus J_{j-1}$ contains no more than $l_j$ vectors, where $l_j$ is the number of $\gamma_r's$ that intersect another $\gamma_{r'}$ with $r \neq r'$ in $p_j$. 
We conclude that 
\begin{align*}
    |\mathcal{H}_{l+1}| &\leq |\mathcal{H}_1 \cup \bigcup_{j=2}^{l+1} (\mathcal{H}_j \setminus \mathcal{H}_{j-1})|\\
    &\leq |J_1 \cup \bigcup_{j=2}^{l+1} J_j|\\
    &\leq |J_1 \cup \bigcup_{j=2}^{l+1} (J_j \setminus J_{j-1})|\\
    &\leq |J_1| + \sum_{j=1}^l l_j \leq m + \frac{m\cdot(m-1)}{2}.
\end{align*}
Where the last inequality follows since the maximum number of intersections between $m$ affine lines is equal to ${m\cdot(m-1)}/{2}$ and by using the trivial estimate of $m$ for $J_1$.
\end{proof}

From Lemma \ref{lem:LemmaOnLambdai} it follows that the number of classifications of $m$ functions $f_1, \dots, f_m$ by $H_{c_1,c_2}(\mathcal{K})$ with $c_1,c_2 \in \R$ is not more than $(m + m\cdot(m-1)/{2}) \cdot (m n + 1)$.
Since $\vcdim(\mathcal{H}(\mathcal{K})) \geq m$, we conclude that 
$$
    2^m \leq \left(m + \frac{m\cdot(m-1)}{2}\right) \cdot (m n + 1)
$$
which, after taking dyadic logarithms, implies 
\begin{align}
    m &\leq \log_2\left(\frac{m\cdot(m+1)}{2}\right) + \log_2(mn + 1)\nonumber\\
    &\leq 2\log_2\left(m+1\right) + \log_2((m+1) n)\nonumber\\
    &\leq  3\log_2(m+1) + \log_2(n).\label{eq:boundForVcDimension}
\end{align}
Since $3 \log_2(m+1) \leq m/2$ for all $m \geq 50$ we conclude that $m < \max\{2 \log_2(n), 50\}$. Hence  $\vcdim(\mathcal{H}(\mathcal{K})) < \infty$. 
Equation \eqref{eq:boundForVcDimension} now yields the claim. 
\end{proof}

We can state two immediate consequences of Theorem \ref{thm:upperBound}, which remove the $\log$ terms in $m$. First, noticing that 
$$
    6 \log(m+1) \leq m, \text{ for all } m \geq 30,
$$
yields the following corollary.

\begin{corollary}\label{cor:upperBoundIfmLarge}
Let $G$ be a group acting on $\mathcal{X}$, let $|G| = n$, and let $\mu$ be the Haar measure on $G$. Let $\mathcal{K} \colon G \to \R$ be a kernel. Then, $\vcdim(\mathcal{H}(\mathcal{K})) \leq \max \{30,  2 \log_2(n)$\}.
\end{corollary}

Based on Corollary \ref{cor:upperBoundIfmLarge}, we can now simplify the estimate of Theorem \ref{thm:upperBound} if the group $G$ is not too small. 

\begin{corollary}
\label{cor:lower_bound_VC_dim}
Let $G$ be a group acting on $\mathcal{X}$, let $|G| = n \geq 16$, and let $\mu$ be the Haar measure on $G$. Let $\mathcal{K} \colon G \to \R$ be a kernel. 
Then, $\vcdim(\mathcal{H}(\mathcal{K})) \leq \log_2(n) + 9 \log_2(\log_2(n))$.
\end{corollary}
\begin{proof}
Thanks to Corollary \ref{cor:upperBoundIfmLarge}, we have that, for $\vcdim(\mathcal{H}(\mathcal{K})) = m$,
\begin{align*}
    m \leq \max\{30, 2\log_2(n)\} \leq 8 \log_2(n)
\end{align*}
if $n \geq 16$. 
Hence Theorem \ref{thm:upperBound} yields that for $n \geq 16$ 
\begin{align*}
    m &\leq \log_2(n) + 3\log_2(m+1)\\
    &\leq \log_2(n) + 3\log_2(8 \log_2(n) + 1)\\
    &\leq \log_2(n) + 3\log_2(16 \log_2(n))  \\    
    &\leq \log_2(n) + 3(\log_2(\log_2(n)) + 4) 
    \leq \log_2(n) + 9 \log_2(\log_2(n)). \qedhere
\end{align*}
\end{proof}

\subsection{Lower bound on VC dimension}

In this subsection, we provide the complement to Theorem \ref{thm:upperBound} in the form of a lower bound on the VC dimension of $\mathcal{H}(\mathcal{K})$ for an appropriately chosen kernel $\mathcal{K}$. The result requires the underlying group $G$ to act on $\mathcal{X}$ with an action $a$ that \emph{has a trivial kernel}, i.e., for an origin $y \in \mathcal{X}$ it holds that $a(g, y) = a(g', y)$ for $g, g' \in G$ only if $g  = g'$.

We state two lower bounds in Theorems~\ref{thm:lowerBound} and \ref{thm:lowerBound_general_case}. The first result uses a specific assumption on the group, which allows a larger lower bound on the VC dimension of $\mathcal{H}(\mathcal{K})$ in terms of the group size. Concretely, we assume the group to contain an element $g \neq e$ of \emph{order two}, i.e. $g\cdot g = e$, where $e$ is the identity element of the group. E.g., for a rotation group, the element that corresponds to a rotation by $\pi$ is an element of order two. 

\begin{theorem}[Lower bound on VC dimension]\label{thm:lowerBound}
Let $m \in \N$, let $G$ be a compact group acting on $\mathcal{X}$ via an action with trivial kernel, let $G$ contain an element of order two, let $|G| \geq 2 m \cdot \binom{m}{\floor{m/2}}$, and let $\mu$ be the Haar measure on $G$.
Then, there is a bounded kernel $\mathcal{K} \colon G \to \R$ such that $\vcdim(\mathcal{H}(\mathcal{K})) \geq m$.
\end{theorem}
\begin{proof}
The theorem follows from the results in Subsections \ref{sec:completeSetsOfOrderings} and \ref{sec:constrOfKernel} below. 
Concretely, in Definition~\ref{def:setOfOrders}, we introduce the notion of a set of orders associated with a kernel and a set of $m\in \N$ functions $F$. Orders are injective maps from $[m]$ to $[m]$. 

For finite groups, we show in Lemma~\ref{lem:ness_suf_cond} that a necessary and sufficient condition for $\mathcal{H}(\mathcal{K})$ to shatter $F$ is that the associated set of orders contains a so-called \emph{complete} set of orders, (Definition \ref{def:complete_set_order}). 
We demonstrate in Lemma \ref{min_complete_ordering} that a complete set of orders of size $\binom{m}{\floor{m/2}}$ exist, let us call it $\mathcal{O}_1$. 
In Lemma~\ref{lem:Constr_permutations}, we provide, for every set $\mathcal{O}$ of $r\in \N$ injective maps from $[m]$ to $[m]$, a set of $m$ functions, and a kernel $\mathcal{K}$ such that the associated set of orders of $\mathcal{H}(\mathcal{K})$ contains $\mathcal{O}$. This requires $|G| \geq 2 r m$. Applying Lemma~\ref{lem:Constr_permutations} with $r =\binom{m}{\floor{m/2}}$ to $\mathcal{O}_1$ yields that $\mathcal{H}(\mathcal{K})$ contains a complete set of orders and finishes the proof if the group is finite.

For infinite groups, the result is shown in Theorem~\ref{thm:infiniteGroupsHaveInfiniteVCDim}.
\end{proof}
In the general case, when the assumption that the group contains an element of order two is dropped, we have the following result.
\begin{theorem}[Lower bound on VC dimension in general case]\label{thm:lowerBound_general_case}
Let $m \in \N$, let $G$ be a compact group acting on $\mathcal{X}$ via an action with trivial kernel, let $|G| \geq 9 m \cdot \binom{m}{\floor{m/2}}$, and let $\mu$ be the Haar measure on $G$.
Then, there is a bounded kernel $\mathcal{K}$ such that $\vcdim(\mathcal{H}(\mathcal{K})) \geq m$.
\end{theorem}
\begin{proof}
The proof of the theorem repeats the proof of Theorem~\ref{thm:lowerBound}, but in the case that the group is finite, we use the Lemma~\ref{lem:Constr_permutations_general} instead of Lemma~\ref{lem:Constr_permutations} to prove that $\mathcal{H}(\mathcal{K})$ contains a complete set of orders. The case of infinite groups is again addressed by Theorem \ref{thm:infiniteGroupsHaveInfiniteVCDim}.
\end{proof}

Next, we state some immediate consequences of Theorem \ref{thm:lowerBound}, which yield a more explicit lower bound on the VC dimension of $\mathcal{H}(\mathcal{K})$.

\begin{corollary}
\label{vs_dim_finite_group}
Let $m \in \N$, let $G$ be a group acting on $\mathcal{X}$ via an action with trivial kernel, let $|G| = n > 1$, and let $\mu$ be the Haar measure on $G$.

Then, there is a bounded kernel $\mathcal{K} \colon G \to \R$ such that $\vcdim(\mathcal{H}(\mathcal{K})) \geq \log_2(n) - 2\log_2(\log_2(n))-4$. Moreover, if $G$ contains an element of order two, then there is a bounded kernel $\mathcal{K} \colon G \to \R$ such that $\vcdim(\mathcal{H}(\mathcal{K})) \geq \log_2(n) - 2\log_2(\log_2(n))-1$.  
\end{corollary}
\begin{proof}
We start with the second part of the assertion, i.e., the case where $G$ has an element of order two. 
Note that, we can assume $n \geq 2^8$ since otherwise 
$$
    \log_2(n) - 2\log_2(\log_2(n))-1 < 1, 
$$
where the result is trivial. 
Wallis' formula~\cite{wastlund2007elementary} yields that for all $m\in \N$
\begin{align}
    \binom{2m}{m} < \frac{4^m}{\sqrt{\pi m}}.
\end{align}

Hence, if 
\begin{align}
\label{condition}
    4m\cdot\binom{2m}{m} <   4m\frac{4^m}{\sqrt{\pi m}} < n, 
\end{align}
then, by Theorem \ref{thm:lowerBound}, there is a bounded kernel $\mathcal{K}\colon G \to \R$ such that $\vcdim(\mathcal{H}(\mathcal{K})) \geq 2m$.
Taking logarithms, we conclude that the second inequality of \eqref{condition} is equivalent to 
\begin{align}\label{eq:logOfCOndition}
 2 + \frac{1}{2} \log_2(m)  + 2m - \frac{1}{2}\log_2(\pi) < \log_2(n).
\end{align}
We pick $m \in \N$, such that $2m \leq \log_2(n) - 2\log_2(\log_2(n))\leq  2m + 1$.
Then,
\begin{align*}
    &2m +  \frac{1}{2} \log_2(m) + 2 - \frac{1}{2}\log_2(\pi)\\
    &\leq\log_2(n) - 2\log_2(\log_2(n)) + \frac{1}{2}\log_2\left(\log_2(n) - 2\log_2(\log_2(n))\right) - \frac{1}{2} + 2 - \frac{1}{2}\log_2(\pi)\\
    &\leq\log_2(n) - 2\log_2(\log_2(n)) + \frac{1}{2}\log_2\left(\log_2(n)\right) +\frac{3}{2} - \frac{1}{2}\log_2(\pi) \\
    &\leq\log_2(n) - \log_2(\log_2(n)) +\frac{3}{2} < \log_2(n),
\end{align*}
which implies \eqref{eq:logOfCOndition} and hence, by \eqref{condition} we conclude that 
\begin{align*}
    \vcdim(\mathcal{H}(\mathcal{K})) \geq 2m \geq \log_2(n) - 2\log_2(\log_2(n))-1.
\end{align*}
The general case follows by a similar argument. We note that we can assume $n \in \N$ to be such that 
\begin{align}\label{eq:ifThisDoesNotHoldTheResultIsTrivial}
    \log_2(n) - 2\log_2(\log_2(n)) - 4 \geq 1
\end{align}
since the result is trivial otherwise. 

Then, instead of \eqref{condition} we use the estimate 
\begin{align}\label{eq:condition2}
    18 m\cdot\binom{2m}{m} <   18 m\frac{4^m}{\sqrt{\pi m}} < n,
    \end{align}
    which yields with Theorem \ref{thm:lowerBound_general_case} the existence of a kernel $\mathcal{K}\colon G \to \R$ such that $\vcdim(\mathcal{H}(\mathcal{K})) \geq 2m$. Then \eqref{eq:condition2}
    holds if
    \begin{align}\label{eq:logOfCOndition2}
         5 + \frac{1}{2} \log_2(m)  + 2m - \frac{1}{2}\log_2(\pi) < \log_2(n).
    \end{align}
    Choosing now $m\in \N$ such that $2m \leq \log_2(n) - 2\log_2(\log_2(n)) - 3\leq 2m + 1$, which is possible because of \eqref{eq:ifThisDoesNotHoldTheResultIsTrivial}, we conclude with the same computation as above, that 
    \begin{align*}
        \vcdim(\mathcal{H}(\mathcal{K})) \geq 2m \geq \log_2(n) - 2\log_2(\log_2(n))-4.
\end{align*}
\end{proof}
In the following subsections, we provide the auxiliary results used to prove Theorems \ref{thm:lowerBound} and \ref{thm:lowerBound_general_case}.

\subsubsection{Complete sets of orders}\label{sec:completeSetsOfOrderings}
To continue further, we need to introduce some additional notation. We will introduce auxiliary variables $\nu_{\mathcal{K}, k}$, which are very closely related to the $\gamma_{i,k}$ of \eqref{eq:theGammas} but defined on the whole domain $\R$ and hence not globally affine. 

\begin{definition}\label{def:setOfOrders}
Let $G$ be a compact group acting on $\mathcal{X}$, let $\mu$ be a finite measure on $G$, let $\mathcal{K}: G \to \R$ be a bounded kernel and let $\{f_1,\dots, f_m\} \subset \mathcal{B}(\mathcal{X}, \mu)$.
We define 
$$
    \nu_{\mathcal{K}, k}:\mathbb{R}\rightarrow\mathbb{R},\ \nu_{\mathcal{K},k}(c) \coloneqq \int_{G} \relu((f_k*_G\mathcal{K})(g) + c)d\mu(g).
$$
For $c \in \R$, we denote by $\sigma(\mathcal{K}, c)$ the \emph{order} of $(\nu_{\mathcal{K},k}(c))_{k=1}^m$, 
    $$
        \sigma(\mathcal{K}, c)(k) \coloneqq 1 + |\{ l \in [m] \colon \nu_{\mathcal{K},l}(c) < \nu_{\mathcal{K},k}(c) \}|.
    $$ 
We denote by $\mathcal{O}(\mathcal{K})\coloneqq \{\sigma(\mathcal{K}, c),\ c\in\mathbb{R}\}$ the \emph{set of orders of $(\nu_{\mathcal{K},k}(c))_{k=1}^m$ obtained by varying $c$}.
\end{definition}

Requiring that $\mathcal{H}(\mathcal{K})$ shatters $\{f_1,\dots, f_m\}$ imposes restrictions on the set of orders $\mathcal{O}(\mathcal{K})$.
Vice versa, there is a set $\mathcal{O} \subset \mathcal{P}([m])$ such that if $\mathcal{O}\subset \mathcal{O}(\mathcal{K})$, then $\mathcal{H}(\mathcal{K})$ shatters $\{f_1,\dots, f_m\}$.
Specifically, we call such sets of orders $\mathcal{O}$ \textit{complete} sets of orders.

\begin{definition}[Complete set of orders]
\label{def:complete_set_order}
Let $m \in \N$. A set $\mathcal{O} \subset \{\sigma \colon [m] \to [m]\}$ is a \emph{complete set of orders} if for each $A \subset [m]$, there exists $\sigma_A\in \mathcal{O}$ such that $\sigma_A(a) < \sigma_A(b)$ for all $a\in A,\ b\in A^c$.
\end{definition}

Next, we relate the shattering properties of G-CNNs for a given set of functions with a certain kernel $\mathcal{K}$ to the property that $\mathcal{O}(\mathcal{K})$ contains a complete set of orders.

\begin{lemma}
\label{lem:ness_suf_cond}
Let $G$ be a group acting on $\mathcal{X}$, let $\mu$ be the Haar measure on $G$, let $\mathcal{K}: G \to \R$ be a bounded kernel and let $\{f_1,\dots, f_m\} \subset \mathcal{B}(\mathcal{X}, \mu)$.

Then, $\mathcal{H}(\mathcal{K})$ contains all functions from $\{f_1,\dots, f_m\}$ to $\{-1,1\}$ if and only if $\mathcal{O}(\mathcal{K})$ contains a complete set of orders.
\end{lemma}
\begin{proof}
First, we prove that if $\mathcal{H}(\mathcal{K})$ shatters $F \coloneqq \{f_1,\dots, f_m\}$, then $\mathcal{O}(\mathcal{K})$ contains a complete set of orders.
Indeed, if $\mathcal{H}(\mathcal{K})$ shatters $F$, then for every $A\subset [m]$ there is a classifier $cl\in\mathcal{H}(\mathcal{K})$, $cl:F \rightarrow \{-1,1\}$, such that $cl(f_j) = -1,$  if $j\in A,$ and $cl(f_j) = 1$  if $f_j\in A^c$.
That means we can pick $c_1,c_2 \in \R$, such that $cl$ is equal to $H_{c_1,c_2}(\mathcal{K})$.
Invoking the definition of $H_{c_1,c_2}(\mathcal{K})$, we conclude that $\nu_{\mathcal{K},j}(c_1) < -c_2$ if $f_j \in A$ and 
$\nu_{\mathcal{K},j}(c_1) \geq -c_2$ if $f_j \in A^c$.
As a consequence $\nu_{\mathcal{K},j}(c_1) > \nu_{\mathcal{K},k}(c_1)$ for all $f_j \in A^c$ and $f_k \in A$. Hence, $\sigma(\mathcal{K}, c_1)(i) < \sigma(\mathcal{K}, c_1)(j)$ for all $i\in A,\ j\in A^c$.

Let us now prove that if $\mathcal{O}(\mathcal{K})$ contains a complete set of orders, then $\mathcal{H}(\mathcal{K})$ shatters $F$.
Let $cl: F \rightarrow \{-1,1\}$ be arbitrary and choose $A = cl^{-1}(\{-1\})$ to be the set of elements where the classifier has value $-1$. 
By assumption, we can find $\sigma\in \mathcal{O}(\mathcal{K})$, such that $\sigma(i) < \sigma(j)$ for all $i\in A,\ j\in A^c$.
As a consequence, there exists $c^*$ such that 
$
    \nu_{\mathcal{K},i}(c^*) < \nu_{\mathcal{K},f_j}(c^*)
$
for all $i \in A$ and $j \in A^c$. 
Therefore, we have that 
$$
    \nu_1 \coloneqq \max_{i \in A}\nu_{\mathcal{K},i}(c^*) < \min_{j \in A^c} \nu_{\mathcal{K}, j}(c^*) \eqqcolon \nu_2.
$$
Setting $c_2 = -(\nu_1 + \nu_2)/2$, we get that 
$$
    H_{c^*,c_2}(\mathcal{K})(f_j) = \left\{ \begin{array}{rl}
      -1 & \text{ if } j \in A \\
      1 & \text{ else. }
      \end{array} \right.
$$
\end{proof}

Based on the connection between $\mathcal{H}(\mathcal{K})$ being shattering to properties of sets of orders established in Lemma \ref{lem:ness_suf_cond}, we can now state a lower bound on the size of $\mathcal{O}(\mathcal{K})$.

\begin{lemma}\label{lem:LowerBoundOnElementsForShattering}
Let $G$ be a group acting on $\mathcal{X}$, let $\mu$ be the Haar measure on $G$, let $\mathcal{K}: G \to \R$ be a bounded kernel and let $\{f_1,\dots, f_m\} \subset \mathcal{B}(\mathcal{X}, \mu)$.

If $\mathcal{H}(\mathcal{K})$ shatters $\{f_1,\dots, f_m\}$, then $\mathcal{O}(\mathcal{K})$ contains at least $\binom{m}{\floor{m/2}}$ elements, where $\floor{m/2}$ is the largest integer part of $m/2$.
\end{lemma}
\begin{proof}
By Lemma~\ref{lem:ness_suf_cond} if $\mathcal{H}(\mathcal{K})$ shatters $F \coloneqq \{f_1,\dots, f_m\}$, then for every set $A \subset [m]$ containing $\floor{m/2}$ elements there is an order $\sigma_A \in \mathcal{O}(\mathcal{K})$ such that $\sigma_A(i) < \sigma_A(j)$ for all $i\in A,\ j\in A^c$.
Moreover, it is easy to see that for two different sets $A$ and $B$ containing $\floor{m/2}$ elements necessarily $\sigma_A \neq \sigma_B$.
Consequently, the number of permutations in $\mathcal{O}(\mathcal{K})$ is not smaller than the number of different sets containing $\floor{m/2}$ elements, which is equal to $\binom{m}{\floor{m/2}}$. 
\end{proof}

The rest of this section shows how to construct a complete set of orders $\mathcal{O}$, containing no more than the necessary number of $\binom{m}{\floor{m/2}}$ elements established in Lemma \ref{lem:LowerBoundOnElementsForShattering}.

\begin{lemma}
\label{fact:2q-1}
Let $S(q,l) \coloneqq \{ A \subset [l] \colon |A| = q\}$ for $q,l \in \N$.
Then, there is a bijective map 
\begin{align*}
    \mathcal{F}_{q,2q-1} : S(q,2q-1)\rightarrow S(q-1, 2q-1)
\end{align*}
such that $\mathcal{F}_{q,2q-1}(A)\subset A$ for every set $A\in S(q,2q-1)$.
\end{lemma}
\begin{proof}
For $q=2$, we define
$$
    F_{2,3}(\{1,2\}) = \{2\},\ F_{2,3}(\{1,3\}) =\{1\},\ F_{2,3}(\{2,3\}) = \{3\}.
$$
Clearly, $F_{2,3}$ satisfies the requirements of the lemma. 

Next, we inductively define $F_{q+1, 2q+1} $, by assuming $F_{q, 2q-1}$ exists with the required property.
For $A\coloneqq\{q, q+1,\dots, 2q+1\}$, we define $F_{q, 2q-1}(A) = \{q+1,\dots,2q+1\}$.
For all other $A \in S(q+1,2q+1)$, we let
\begin{align} \label{eq:constructionOfj}
    j &\coloneqq \max\{ a\in A\text{ such that } a+1\notin A\}.
\end{align}
Define 
\begin{align}\label{eq:defOfF}
    F_{q+1, 2q+1}(A) \coloneqq \{j\} \cup T_{j}^{-1}(F_{q, 2q-1}(T_{j}(A\setminus\{j\})),
\end{align}
where 
$$
 T_{j}(q) \coloneqq \left \{ \begin{array}{rl}
    q  & \text{ if }  q < j\\
    q-2  & \text{ if } q > j+1.\\
 \end{array}\right.
$$
Per induction assumption, it holds that $F_{q+1, 2q+1}(A) \subset A$ since $(F_{q, 2q-1}(T_j(A\setminus\{j\})) \subset T_j(A\setminus\{j\})$.

Now we show that $j$, which is defined in \eqref{eq:constructionOfj} is equal to $\max\{i\in F_{q+1, 2q+1}(A)\text{ such that } i+1\notin F_{q+1, 2q+1}(A)\}$.
First we note, that by construction $j$ is in $F_{q+1, 2q+1}(A)$.
Moreover, as $F_{q+1, 2q+1}(A)\subset A$ and $j+1\notin A$, it holds that $j+1\notin F_{q+1, 2q+1}(A)$.
We will prove the maximality of such $j$ by contradiction.
Let 
\begin{align*}
    j_1 &\coloneqq \min\{ j'> j \colon j'\in F_{q+1, 2q+1}(A),\text{ and } j'+1\notin F_{q+1, 2q+1}(A)\},\\
    A_1&\coloneqq A\setminus\{j\}.
\end{align*}
Then, it follows from \eqref{eq:defOfF} that 
$$
j_1 \in T_{j}^{-1}(F_{q, 2q-1}(T_{j}(A_1)) \text{ and } j_1+1 \not\in T_{j}^{-1}(F_{q, 2q-1}(T_{j}(A_1)).
$$
Hence, $j_1-2 = T_j(j_1)\in F_{q, 2q-1}(T_j(A_1))$ and $j_1 - 1 = T_j(j_1+1)\notin F_{q, 2q-1}(T_j(A_1))$.
Consequently,
$$
j_2\geq T_j(j_1),\text{ where } j_2\coloneqq \max\{i\in F_{q, 2q-1}(A_1)\text{ such that } i+1\notin F_{q, 2q-1}(T_j(A_1))\}.
$$
By induction assumption $j_2 = \max\{ a\in T_j(A_1)\text{ such that } a+1\notin T_j(A_1)\}$.
Consequently, $j_2+2 = T_j^{-1}(j_2)\in A_1\subset A$ and $T_j^{-1}(j_2+1)=(T_j^{-1}(j_2)+1)\notin A_1$.
Therefore, $T_j^{-1}(j_2) \geq j_1 > j$ and $T_j^{-1}(j_2)\in A,\ (T_j^{-1}(j_2)+1)\notin A$, which contradicts the definition of $j$.

\smallskip
Now we are ready to show the injectivity of $F_{q+1, 2q+1}$. Let $A,B \in S(q+1,2q+1)$ with $A \neq B$ such that  
$$
    F_{q+1, 2q+1}(A) =  F_{q+1, 2q+1}(B).
$$
As $A\neq B$ one of these sets not equal to $\{q,\dots,2q+1\}$.
Without loss of generality, we assume that $A\neq\{q+1,\dots,2q+1\}$.
Then, we can define $j_A$ as in \eqref{eq:constructionOfj}.
By the statement above it holds that 
\begin{align*}
j_A &= \max\{i\in F_{q+1, 2q+1}(A)\text{ such that } i+1\notin F_{q+1, 2q+1}(A)\}\\
&=\max\{i\in F_{q+1, 2q+1}(B)\text{ such that } i+1\notin F_{q+1, 2q+1}(B)\}\\
&=\max\{ b\in B\text{ such that } b+1\notin B\} = j_B.
\end{align*}
That means that $A\setminus\{j_A\}\neq B\setminus\{j_B\}$, but, by \eqref{eq:defOfF}, 
$$
F_{q, 2q - 1}(T_{j_A}(A\setminus\{j_A\})) = F_{q, 2q - 1}(T_{j_B}(B\setminus\{j_B\})).
$$
The last equation is a contradiction to the induction assumption.
Since $|S(q,2q-1)| = |S(q-1,2q-1)|$, it follows from the injectivity that $\mathcal{F}_{q,2q-1}$ is also surjective.
\end{proof}

\begin{lemma}
\label{fact_maps}
Let $m \in \N$.
Then, for all $q\in[m]$ there exists a map  $\mathcal{F}_{q,m} : S(q,m)\to S(q-1, m)$ such that for every $A\in S(q,m)$, it holds that $\mathcal{F}_{q,m}(A)\subset A$.
Moreover, if $\binom{m}{q}\geq\binom{m}{q-1}$, then $\mathcal{F}_{q,m}$ is surjective.
If $\binom{m}{q}\leq\binom{m}{q-1}$, then $\mathcal{F}_{q,m}$ is injective.
\end{lemma}
\begin{proof}
We prove this statement by induction over $m \in \N$.
It is easy to see that the statement holds for all $q\leq m$ for $m=3$.
Let us assume that the statement holds for all $m<r$; we will prove the statement for $m=r$.
We denote 
\begin{align*}
    S_{q,m} \coloneqq\{A\in S(q,m) \colon m\in A\},
\end{align*}
and
\begin{align*}
S_{q,m}^c \coloneqq\{A\in S(q,m) \colon m\notin A\}.  
\end{align*}
Let $\mathcal{P}([m])$ be the power set of $\{1,2,\dots,m\}$, we define three maps:
\begin{align*}
    &\mathrm{Pr}_m: \mathcal{P}([m]) \to \mathcal{P}([m-1]),&\ &\mathrm{Pr}_m(A) = \{a\in A,\ a<m\},\\
    &\mathrm{Inv}_m:\ \mathcal{P}([m-1])\to \mathcal{P}([m]),&\ &\mathrm{Inv}_m(A) = A,\\
&\widehat{\mathrm{Inv}}_m: \mathcal{P}([m-1])\to \mathcal{P}([m]),&\ &\widehat{\mathrm{Inv}}_m(A) = A\cup\{m\}.
\end{align*}
We prove the statement of Lemma \ref{fact_maps} by considering three cases.

\smallskip
\textbf{Case 1:} $m\geq 2q$.
In this case, it holds that
\begin{align*}
 \binom{m-1}{q}\geq\binom{m-1}{q-1} \text{ and } \binom{m-1}{q-1} > \binom{m-2}{q-1}.
\end{align*}
Consequently, by the induction hypothesis, there exists 
$\mathcal{F}_{q,m-1}: S(q,m-1)\rightarrow S(q-1,m-1),$
such that for all $A\subset S(q,m-1)$ it holds that $\mathcal{F}_{q,m-1}(A)\subset A$, and $\mathcal{F}_{q,m-1}$ is surjective.

Also, by the induction hypothesis, there exists $
    \mathcal{F}_{q-1,m-1}: S(q-1,m-1)\rightarrow S(q-2,m-1),
$
such that for all $A$ in $S(q-1,m-1)$ it holds that $\mathcal{F}_{q-1,m-1}(A)\subset A$, and $\mathcal{F}_{q-1,m-1}$ is surjective.
We define $\mathcal{F}_{q,m}(A)$ using $\mathcal{F}_{q-1,m-1}$ and $\mathcal{F}_{q,m-1}$ by 

\begin{align*}
        \mathcal{F}_{q,m}(A) \coloneqq \left\{ \begin{array}{ll}
      \mathrm{Inv}_m\circ\mathcal{F}_{q,m-1}\circ \mathrm{Pr}_m(A) & \text{ if } A\in S_{q,m}^{c}\\
      \widehat{\mathrm{Inv}}_m\circ\mathcal{F}_{q-1,m-1}\circ \mathrm{Pr}_m(A) & \text{ if } A\in S_{q,m}.
      \end{array} \right.
\end{align*}
 It is clear that $\mathcal{F}_{q,m}$ satisfies all the requirements of the theorem by construction.
  
\smallskip
\textbf{Case 2:} $m =  2q-1$.
This case was proved in Lemma~\ref{fact:2q-1}.

\smallskip
\textbf{Case 3:} $m\leq 2q-2$.
In this case, it holds that
\begin{align*}
 \binom{m-1}{q} < \binom{m-1}{q-1} \text{ and } \binom{m-1}{q-1} \leq \binom{m-2}{q-1}.
\end{align*}

By the induction hypothesis, there exists 
$
    \mathcal{F}_{q,m-1}: S(q,m-1)\rightarrow S(q-1,m-1),
$
such that for all $A\subset S(q,m-1)$ it holds that $\mathcal{F}_{q,m-1}(A)\subset A$, and $\mathcal{F}_{q,m-1}$ is injective.

Also, by the induction hypothesis, there exists 
$
    \mathcal{F}_{q-1,m-1}: S(q-1,m-1)\rightarrow S(q-2,m-1),
$
such that for all $A \subset S(q-1,m-1)$ it holds that $\mathcal{F}_{q-1,m-1}(A)\subset A$, and $\mathcal{F}_{q-1,m-1}$ is injective.
We define $\mathcal{F}_{q,m}(A)$ using $\mathcal{F}_{q-1,m-1}$ and $\mathcal{F}_{q,m-1}$ by
\begin{align*}
        \mathcal{F}_{q,m}(A) \coloneqq \left\{ \begin{array}{ll}
      \mathrm{Inv}_m\circ\mathcal{F}_{q,m-1}\circ \mathrm{Pr}_m(A) & \text{ if } A\in S_{q,m}^{c}\\
      \widehat{\mathrm{Inv}}_m\circ\mathcal{F}_{q-1,m-1}\circ \mathrm{Pr}_m(A) & \text{ if } A\in S_{q,m}.
      \end{array} \right.
\end{align*}
It holds that $\mathcal{F}_{q,m}$ satisfies all the requirements of the theorem by construction.
\end{proof}

In the following lemma, we now establish the existence of relatively small complete sets of orders. The proof of this lemma is based on Lemma~\ref{fact_maps}.
More specifically, the maps $(\mathcal{F}_{q,m})_{q=1}^{m}$ are used to define an order on $\mathcal{P}([m])$.

\begin{lemma}
\label{min_complete_ordering}
Let $m \in \N$. There is a complete set of orders $\mathcal{O}$, containing no more than $\binom{m}{\floor{ m /2 }}$ elements.
\end{lemma}
\begin{proof}
Let $A\in S(q,m)$. We define $F_{1,m} \coloneqq \mathcal{F}_{q,m}$ and for $k\in\{2,\dots, q\}$, 
\begin{align*}
    F_{k,m} \coloneqq \mathcal{F}_{q-k+1,m}\circ\cdots\circ\mathcal{F}_{q,m}.   
\end{align*}
We define the order $\tilde{\sigma}_A$ on $A$ by 
\begin{align}
   \tilde{\sigma}_A:\ A&\to\{1,2,\dots, q\},\nonumber\\
   \tilde{\sigma}_A(a) &= k,\ \text{ if } a \in F_{k-1,m}(A)\setminus F_{k,m}(A) 
   \text{ for } k\in\{1,\dots,q\},\label{sigma} 
\end{align}
where $F_{0,m}(A)\coloneqq A$. We have that $\tilde{\sigma}_A$ is well defined since, for $k \in [q]$, if $a \in F_{k-1,m}(A)\setminus F_{k,m}(A)$, then $a \not \in F_{k,m}(A)$ and hence $a \not \in F_{k',m}(A)$ for every $k' > k$. Similarly, since $a \in F_{k-1,m}(A)$ it follows that $a \in F_{k'',m}(A)$ for all $k'' < k$. Hence, there exists exactly one $k \in [q]$ such that $a \in F_{k-1,m}(A)\setminus F_{k,m}(A)$.

We note that by construction:
\begin{align}
\label{ordering_property}
\tilde{\sigma}_A(a) > \tilde{\sigma}_A(b),\text{ for every }a\in F_{k,m}(A),\ \text{and every }b\in A\setminus F_{k,m}(A). 
\end{align}

We say that an order $\sigma$ \textit{separates} a set $B_1$ from a set $B_2$ if 
\begin{align}
\label{eq:separation}
\sigma(a) > \sigma(b),\text{ for every }a\in B_1,\ \text{and every }b\in B_2. 
\end{align}
In particular, \eqref{ordering_property} yields that $\tilde{\sigma}_A$ separates $F_{k,m}(A)$ from $A\setminus F_{k,m}(A)$.

We construct a complete set of orders $\mathcal{O} = \bigcup_{i = 0}^{q} \mathcal{O}_{m-i}$, with 
$$
    \mathcal{O}_{m} \subset \mathcal{O}_{m-1} \subset \cdots \subset \mathcal{O}_{m-q},
$$
where $q = \floor{m/2}$.
We start by defining for $A = [m] \in S(m,m)$ the set of orders $\mathcal{O}_m = \{\tilde{\sigma}_A\}$.

Next, we assume that, in the $i$-th step, for each subset $A\in S(m-i+1,m)$, there is exactly one order $\sigma \in \mathcal{O}_{m-i+1}$ that separates $A$ from $A^c \coloneqq [m] \setminus A$.
We proceed with the $i+1$-st step. 
We note, that as $i \leq \floor{m/2}$ it holds that $|S(m-i+1,m) | < | S(m-i,m) |$.
We define $\mathcal{O}_{m-i}$ as the union of $\mathcal{O}_{m-i+1}$ with the following orders: for each $A\in S(m-i,m)$, such that there is no order in $\mathcal{O}_{m-i+1}$ separating $A$ from $A^c$, we construct the following order
\begin{align} \label{eq:constructionsOforderings}
    \sigma_A:\{1,2,\dots,m\}&\to \{1,2,\dots,m\},\\
    \sigma_A(a) &= \tilde{\sigma}_A(a),\text{ for every }a\in A,\ \text{ and }\nonumber \\
    \sigma(a') &= i-1+\tilde{\sigma}_{A^c}(a'), a' \in A^c.\nonumber   
\end{align}

We make the following two observations about $\mathcal{O}_{m-i}$:
\begin{enumerate}
    \item Each $\tilde{A}\in S(m-i,m)$ is separated from $\tilde{A}^c$ by some $\sigma\in \mathcal{O}_{m-i}$.
    \item If $\tilde{A}\in S(m-i,m)$ is separated from $\tilde{A}^c$ by $\sigma\in \mathcal{O}_{m-i}$, then $\sigma$ is of the form \eqref{eq:constructionsOforderings} for some $Q \subset [m]$ with $|Q| \geq m-i$.
\end{enumerate}

Next, we will show that for $\bar{A} \in S(m-i,m)$, the separating order is unique. Assume that $\bar{A} \in S(m-i,m)$ is separated from $\bar{A}^c$ by $\sigma$ and $\sigma'$. By the previous observation, we have that $\sigma = \sigma_U$ and $\sigma' = \sigma_Q$ for some $U,Q \subset [m]$ with $|U|, |Q| \geq m-i$. 

Assume towards a contradiction that $\sigma_Q \neq \sigma_U$. Then, we conclude that $\sigma_V$ separates for $V = Q, U$ the sets $\sigma_V^{-1}( \{1, \dots, m-i+1\})$ from $\sigma_V^{-1}( \{m-i+2, \dots, m\})$. By induction assumption, we have that 
\begin{align}\label{eq:contradictAgainstThis}
    \sigma_Q^{-1}( \{1, \dots, m-i+1\}) \neq \sigma_U^{-1}( \{1, \dots, m-i+1\}).
\end{align}
However, by \eqref{sigma} it holds that 
\begin{align*}
     \mathcal{F}_{m-i+1, m}(\sigma_Q^{-1}( \{1, \dots, m-i+1\})) = A = \mathcal{F}_{m-i+1, m}(\sigma_U^{-1}( \{1, \dots, m-i+1\})).
\end{align*}
Since $\mathcal{F}_{m-i+1, m}$ is injective, we arrive at a contradiction to \eqref{eq:contradictAgainstThis}. This yields that $\sigma_Q = \sigma_U$, and hence we observe that for each $A\in S(m-i,m)$, there is exactly one order from $\mathcal{O}_{m-i}$, that separates it. Hence, we conclude that $|\mathcal{O}_{m-i+1}|=|S(m-i,m)| = \binom{m}{m-i}$. In particular, $|\mathcal{O}_{m-q+1}| = \binom{m}{\ceil{m /2}} = \binom{m}{\floor{m /2}}$.

Now, we will prove that every set $A\in \mathcal{P}([m])$ is separated from $A^c$ by some order from $\mathcal{O}_{m-q}$.
We prove this statement by contradiction.

Let $A\in S(k,m)$ with the biggest $k$ such that $A$ is not separated from $A^c$ by an order from $\mathcal{O}_{m-q}$. By the previous part of the proof, we can assume that $k \geq q$.

In this case, as $|S(k+1,m) | > | S(k,m)|$ it holds that $\mathcal{F}_{k+1,m}$ is surjective by Lemma \ref{fact_maps}.
Therefore, there exists $B\in S(k+1,m)$, such that $\mathcal{F}_{k+1,m}(B) = A$. 
Moreover, as $k+1>k$ it holds that $B$ is separated from $B^c$ by $\sigma\in\mathcal{O}_{m-q}$.
By construction, this implies that $A$ is separated from $A^c$ by $\sigma$, which produces a contradiction. 

We conclude that $\mathcal{O}$ is a complete set of orders, containing $\binom{m}{\floor{m/2}}$ elements.
\end{proof}

\subsubsection{Construction of an expressive kernel}\label{sec:constrOfKernel}

Let $G$ be a group with $|G| \geq n$. In this subsection, we construct a kernel $\mathcal{K}: G \to \R$ such that $\vcdim(\mathcal{H}(\mathcal{K}))$ is close to the upper bound of Theorem \ref{thm:upperBound}. 
We construct for every $m \in \N$ with $2 m \cdot\binom{m}{\floor{m/2}} \leq n$ a bounded kernel $\mathcal{K}$, such that $\vcdim(\mathcal{H}(\mathcal{K})) \geq m$.
We also identify the functions $f_1,\dots, f_m$ that are shattered by $\mathcal{H}(\mathcal{K})$.

Recall that an action has trivial kernel if for an origin $y \in \mathcal{X}$ as in \eqref{eq:defOfFbar} the action $a$ on $G$ satisfies that $a(g,y) = a(g', y)$ only if $g = g'$. Therefore, for a function $u \colon G \to \R$, we can define 
\begin{align*}
    f \colon \mathcal{X} &\to \R \\
    f(x) &= \left\{ \begin{array}{ll}
        u(g) &\text{ if } x = a(g, y) \text{ for a } g \in G\\
        0 & \text{ else.}
    \end{array}\right.
\end{align*}
It is not hard to see that, in this case, $\bar{f} = u$. Hence, if we find $m\in \N$ real-valued functions $u_1, \dots, u_m$ on $G$ which are shattered by
\begin{align*}
    \left\{ u \mapsto \mathrm{sign}\left(\int_G \relu((u *_G\mathcal{K})(g) + c_1) d\mu(g) + c_2\right) \colon c_1, c_2 \in \R\right\},
\end{align*}
then there also exist real-valued functions $f_1, \dots, f_m$ on $\mathcal{X}$ that are shattered by $\mathcal{H}(\mathcal{K})$.

Specifically, for $g \in G$ with $g \neq e$, we will choose the functions $u_1,\dots, u_m$ from the two-dimensional space of functions $U$ generated by the functions $\mathds{1}_e,\ \mathds{1}_g$.
In this case, every function $u\in U$ can be written as a linear combination of $\mathds{1}_e,\ \mathds{1}_g$:
\begin{align}\label{eq:decompositionOfAU}
    u\coloneqq a  \mathds{1}_e + b  \mathds{1}_g,\ a,b\in\mathbb{R}.
\end{align}
To simplify the notation, we define for $u$ as in \eqref{eq:decompositionOfAU}, 
\begin{align*}
    \widetilde{u}:\mathbb{R}^2\rightarrow \mathbb{R},\ \widetilde{u}(\mathbf{k}) = k_1   a + k_2   b,\text{ where } \mathbf{k}=(k_1,k_2).
\end{align*}

In the following lemma, we construct for $A >0$ functions $u_0, \dots, u_{2m} \in U$ and associated vectors $\mathbf{k}^{2i, A} \in \R^2$, $i \in [m]$ such that $\widetilde{u_{2i}}(\mathbf{k}^{2i, A}) = A$ and $\widetilde{u_{l}}(\mathbf{k}^{2i, A})< B < A$ for all $l \neq 2i$, $l\in [2m] \cup \{0\}$. This choice of $u_0, \dots, u_{2m}$ and $\mathbf{k}^{i, A}$ is essential for the construction of a kernel that scatters a subset of $u_0, \dots, u_{2m}$.

\begin{lemma}
\label{example}
Let $m \in \N$, let $G$ be a group acting on $\mathcal{X}$, and let $g \in G$ be such that $g \neq e$. Let $C>B > 0$, $p \in \N$, and let for $i \in [p]$
\begin{align*}
    \epsilon_i &\coloneqq 4 \frac{C}{B} + 1 + (p-i)  \left(\frac{B}{C} + 1\right).
\end{align*}
For $i \in [2p + 1]$, let $u_i \colon G \to \R$ be defined by 
\begin{align*}
   u_0 &\coloneqq \mathds{1}_{e}, \\ 
   u_1 &\coloneqq \mathds{1}_{g}, \\
   u_2 &\coloneqq \epsilon_1  u_0 + u_1\eqqcolon a_{2, 1}  \mathds{1}_{e} + a_{2, 2}  \mathds{1}_{g}, \\
   u_3 &\coloneqq  u_0 + \epsilon_1  u_1\eqqcolon a_{3, 1}  \mathds{1}_{e} + a_{3, 2}  \mathds{1}_{g},\\
   u_{2q} &\coloneqq \epsilon_{q}  u_{2q - 2} + u_{2q - 1}\eqqcolon a_{2q, 1}  \mathds{1}_{e} + a_{2q, 2}  \mathds{1}_{g}, \\
   u_{2q + 1} &\coloneqq u_{2q - 2} + \epsilon_{q}  u_{2q - 1}\eqqcolon  a_{2q+1, 1}  \mathds{1}_{e} + a_{2q+1, 2}  \mathds{1}_{g}, \text{ for } q \in [p].
\end{align*}
 
Then, for all $i= 2,4, \dots,2 p$ and for all $A$ such that $C > A > B$, there exist $\mathbf{k}^{i, A}\in\mathbb{R}^2$ such that 
\begin{align*}
    \widetilde{u_i}(\mathbf{k}^{i, A}) = A \text{ and }\widetilde{u_l}(\mathbf{k}^{i, A}) < B, \text{ for all } l\in [2p+1] \text{ with }  l\neq i.
\end{align*}
\end{lemma}
\begin{proof}
Note that the definition of $\epsilon_i$ 
implies that 
\begin{align}
    \epsilon_i &> \frac{4C}{B} > \frac{2C}{B} + \frac{1}{2} > 2, \text{ for all } i \in [p], \label{eq:FirstPropertiesOfEpsilon}\\
    \epsilon_{i+1} &< \epsilon_i - \frac{B}{C}, \text{ for all }  i \in [p-1].\label{eq:SecondPropertyOfEpsilon}
\end{align}
Let $A\in (B,C)$ and $m\in [p]$. We now show how to find $\mathbf{k}^{m,A}$. Note first that for all $\alpha, \beta \in \R$
\begin{align}
    \alpha u_{2m - 2} + \beta u_{2 m - 1} &= \alpha \epsilon_{m-1}u_{2m - 4} + \alpha  u_{2m - 3} + \beta  u_{2m - 4} + \beta \epsilon_{m-1} u_{2m - 3} \label{eq:convenientFormOfsumoffs}\\
    &= (\alpha\epsilon_{m-1} + \beta)u_{2m - 4}  + (\alpha + \beta\epsilon_{m-1}) u_{2m - 3}. \nonumber
\end{align}
From $\epsilon_{m-1} > 2$, it follows that $\left(\alpha\epsilon_{m-1} + \beta,\alpha + \beta\epsilon_{m-1}\right) \neq (0,0)$ for all $(\alpha, \beta) \neq 0$.
As a consequence, $\alpha u_{2m - 2} + \beta u_{2 m - 1} = 0$ for some $(\alpha, \beta)\neq (0,0)$ if and only if $\alpha' u_{2m - 4} +  \beta' u_{2 m - 3} = 0$ for some $(\alpha', \beta')\neq (0,0)$. Since $\alpha u_0 + \beta u_1 = 0$ holds for no tuple $(\alpha, \beta) \neq (0,0)$, we conclude that $\alpha u_{2m - 2} + \beta u_{2 m - 1} \neq 0$ for all $\alpha, \beta \neq (0,0)$.

As a consequence, we have that $\alpha (a_{2m - 2, 1}, a_{2m - 2, 2}) \neq \beta (a_{2m - 1, 1}, a_{2m - 1, 2})$ for all $(\alpha, \beta) \neq (0,0)$ and we conclude that the matrix 
\begin{align*}
    Q \coloneqq \left(\begin{array}{cc}
       a_{2m - 2, 1}   & a_{2m - 2, 2}  \\
       a_{2m - 1, 1}   & a_{2m - 1, 2} 
    \end{array}\right)
\end{align*}
has full rank. Note that by construction, for $\mathbf{k} \in \R^2$ it holds that 
$$
Q \mathbf{k} = \binom{\widetilde{u}_{2 m - 2}(\mathbf{k})}{\widetilde{u}_{2 m - 1}(\mathbf{k})}.
$$ 
As a consequence, we have that there exist $\mathbf{k} = (k_1^{m,A},k_2^{m,A}) \in \R^2$ such that  
\begin{align}\label{eq:constructionOfK1}
    \widetilde{u}_{2 m - 2}(\mathbf{k}) &= 2  \epsilon_m^{-1}  A,\\
    \widetilde{u}_{2 m - 1}(\mathbf{k}) &= -A.\label{eq:constructionOfK2}
\end{align}
Due to the first inequality in \eqref{eq:FirstPropertiesOfEpsilon}, we conclude because of $A < C$ that 
\begin{align*}
    \widetilde{u}_{2 m - 2}(\mathbf{k}) < 2 \epsilon_m^{-1} C < \frac{B}{2}.
\end{align*}
We note that, by \eqref{eq:convenientFormOfsumoffs} with $\alpha = \epsilon_{m-1}$ and $\beta = -1$, 
\begin{align*}
    \epsilon_{m-1}  u_{2 m -2} - u_{2 m -1} &= \epsilon_{m-1}^2  u_{2 m -4} + \epsilon_{m-1} u_{2 m -3} - u_{2 m - 4} - \epsilon_{m-1}u_{2 m -3}\\
&=\epsilon_{m-1}^2  u_{2 m -4}  - u_{2 m - 4}\\
&=(\epsilon_{m-1}^2-1)  u_{2 m -4}.\\
\end{align*}
Therefore, 
\begin{align}\label{eq:ioasuhdfuiashd}
    u_{2 m -4} = \frac{1}{\epsilon_{m-1}^2-1} \left(\epsilon_{m-1}  u_{2 m -2} - u_{2 m -1}\right).
\end{align}
Since $\mathds{1}_{e}$ and $\mathds{1}_{g}$ are linearly independent, \eqref{eq:ioasuhdfuiashd} implies that $a_{2 m -4, \iota} = (\epsilon_{m-1}^2-1)^{-1}   \left(\epsilon_{m-1}  a_{2 m -2, \iota} - a_{2 m -1, \iota}\right)$ for $\iota = 1,2$. Hence, 
\begin{align}\label{eq:f2mMinus4}
    \widetilde{u}_{2 m -4}(\mathbf{k}) = \frac{1}{\epsilon_{m-1}^2-1} \left(\epsilon_{m-1}  \widetilde{u}_{2 m -2}(\mathbf{k}) - \widetilde{u}_{2 m -1}(\mathbf{k})\right).
\end{align}
Moreover, we have from \eqref{eq:FirstPropertiesOfEpsilon} that $\epsilon_{m-1} > 2C/B + 1/2 > 2C/B + 1/\epsilon_{m-1}$ and hence it immediately follows that 
\begin{align} \label{eq:EstimateOfEpsilonMMinus1}
   \frac{\epsilon_{m-1} C}{\epsilon_{m-1}^2 - 1} < \frac{B}{2}.
\end{align}
Plugging \eqref{eq:constructionOfK1}, \eqref{eq:constructionOfK2}, and \eqref{eq:EstimateOfEpsilonMMinus1} into \eqref{eq:f2mMinus4} yields
\begin{align*}
    \widetilde{u}_{2 m - 4}(\mathbf{k}) &= \frac{1}{\epsilon_{m-1}^2-1} \left(\epsilon_{m-1} \widetilde{u}_{2 m -2}(\mathbf{k}) - \widetilde{u}_{2 m -1}(\mathbf{k})\right)\\
    &= \frac{1}{\epsilon_{m-1}^2-1} \left(\epsilon_{m-1} 2 \epsilon_{m}^{-1}A   + A\right)\\
    &< B \epsilon_{m}^{-1}  + \frac{B}{2} < B.
\end{align*}
Also, by a similar argument as above $|\widetilde{u}_{2 m - 3}(\mathbf{k})| < B$.

Next, we note, that if $|\widetilde{u}_{2j}(\mathbf{k})| < B$ and $|\widetilde{u}_{2j+1}(\mathbf{k})| < B$, then $|\widetilde{u}_{2t}(\mathbf{k})| < B$ and $|\widetilde{u}_{2t+1}(\mathbf{k})| < B$ for all $t<j$. 
Indeed, $|\widetilde{u}_{2t}(\mathbf{k})| < \frac{\epsilon_{t} |\widetilde{u}_{2t+2}(\mathbf{k})| + |\widetilde{u}_{2t+3}(\mathbf{k})|}{\epsilon_{t}^2 - 1} < \frac{B}{\epsilon_{t} - 1} < B$. 
The same argument can be made to bound $|\widetilde{u}_{2t+1}(\mathbf{k})|$ by $B$. 
Thus, $|\widetilde{u_i}(\mathbf{k})| < B$ for all $i<2m - 2$.

Let us collect what we have proved so far:
\begin{enumerate}
    \item It holds that $|\widetilde{u_i}(\mathbf{k})| < B$ for all $i<2m - 2$.
    \item For $i = 2m - 2$ or $i = 2m - 1$, it follows directly from \eqref{eq:constructionOfK1} and \eqref{eq:constructionOfK2}, that $\widetilde{u_i}(\mathbf{k}) < B$.
    \item For $i = 2m$, it holds that $\widetilde{u_i}(\mathbf{k}) = \epsilon_i 2 \epsilon_i^{-1} A  - A = A$.
    \item For $i = 2m+1$, it holds that $\widetilde{u_i}(\mathbf{k}) = 2 \epsilon_i^{-1} A  - \epsilon_i A < 0$ by \eqref{eq:FirstPropertiesOfEpsilon}.
\end{enumerate}
Hence, the proof is complete if we show that $\widetilde{u_i}(\mathbf{k}) < 0$ for all $i > 2m+1$.
We have that
\begin{align*}
    \widetilde{u}_{2m+2}(\mathbf{k}) = \epsilon_{m+1} A + 2\epsilon_{m}^{-1} A - \epsilon_{m}A.
\end{align*}
Moreover, 
\begin{align*}
   \epsilon_{m+1} A  - \epsilon_{m} A = (\epsilon_{m+1} - \epsilon_{m}) A < -\frac{B}{C} A.
\end{align*}
Consequently, 
\begin{align*}
    \widetilde{u}_{2m+2}(\mathbf{k}) < -\frac{B}{C} A + 2\epsilon_{m}^{-1} A < 0,
\end{align*}
since $2 \epsilon_{m}^{-1} - B/C < B/(2C) -B/C < 0$ by \eqref{eq:FirstPropertiesOfEpsilon}. Moreover,
\begin{align*}
    \widetilde{u}_{2m+3}(\mathbf{k}) = A + 2\epsilon_{m}^{-1}\epsilon_{m+1} A - \epsilon_{m}\epsilon_{m+1}A < A - \epsilon_{m+1}A < 0.
\end{align*}

Finally, it is not hard to see that, if $\widetilde{u}_{2m+3}(\mathbf{k})<0$ and $\widetilde{u}_{2m+2}(\mathbf{k})<0$, then $\widetilde{u}_{i}(\mathbf{k})<0$ for all $i>2m+3$. 
\end{proof}

Next, we show that for $m  \in \N$ and a given set of orders $\mathcal{O} \subset \{ \sigma \colon [m] \to [m]\}$, if the group $G$ can be partitioned in a specific way, then there exists a kernel such that for the functions of Lemma~\ref{example}, an associated set of orders defined similarly to Definition~\ref{def:setOfOrders} contains $\mathcal{O}$. 

\begin{lemma}\label{lem:prepre_Constr_permutations}
Let $C > B >0$ and $r, m \in \N$. Let $\mathcal{O} = \{o_1,\dots, o_r\} \subset \{ \sigma : [m] \to [m]\}$, let $G$ be a group and $g \in G$ such that $g \neq e$, and let $\mu$ be a finite measure on $G$ which satisfies that $\mu(g) = \mu(e) = 1$.
In addition, for $l \in [r]$, assume that $G$ contains subsets $H_l \subset G$ satisfying
\begin{enumerate}
    \item $|H_l| = m$,
    \item $H_j \cap H_l = \emptyset$ for all $j \in [r]$ such that $j \neq l$,
    \item $H_j \cap (g^{-1}\cdot H_l) = \emptyset$ for all $j \in [r]$ such that $j \neq l$,
    \item $H_l = \{h_{l,1},\dots, h_{l,m}\}$, where $h_{l,i}\neq g^{-1}\cdot h_{l, j}$ for all $j,i \in [r]$ such that $j \neq i$.
\end{enumerate}
We define $u_0,\dots, u_{2m}$ as in Lemma~\ref{example}, $\widetilde{H}_r  \coloneqq \bigcup_{l \leq r}H_l$, we set
$$
    \hat{\nu}_{\mathcal{K}, k}:\mathbb{R}\rightarrow\mathbb{R},\ \hat{\nu}_{\mathcal{K},k}(c)= \sum_{g \in \widetilde{H}_r} \relu((u_{2k}*_G\mathcal{K})(g) + c),
$$
and denote the order of $(\hat{\nu}_{\mathcal{K},k}(c))_{k=1}^m$ by 
    $$
        \hat{\sigma}(\mathcal{K}, c)(k) \coloneqq 1 + |\{ l \in [m] \colon \hat{\nu}_{\mathcal{K},l}(c) < \hat{\nu}_{\mathcal{K},k}(c) \}|.
    $$ 
We denote by $\widehat{\mathcal{O}}_B(\mathcal{K})\coloneqq \{\hat{\sigma}(\mathcal{K}, c),\ c < B\}$ the \emph{set of orders of $(\hat{\nu}_{\mathcal{K},k}(c))_{k=1}^m$ obtained by varying $c<B$}.
Then, there exists a bounded kernel $\mathcal{K}:G\rightarrow\mathbb{R}$ such that $\mathcal{O}\subset \widehat{\mathcal{O}}_B(\mathcal{K})$.
\end{lemma}

\begin{proof}
We fix 
\begin{align}
    \label{def_epsilon}
    \epsilon \coloneqq \frac{(C-B)\cdot (m-1)}{2m\cdot (m - 1 + m^{r+1} - 1)}.
\end{align}

We will construct the kernel $\mathcal{K}: G \to \R$ sequentially. 
We start by defining the kernel $\mathcal{K}_0: G \to \R$ to be $0$ on all elements of the group $G$.

Next, we define the kernel $\mathcal{K}$, by iteratively updating $\mathcal{K}_l$ to yield $\mathcal{K}_{l+1}$ and ultimately setting $\mathcal{K} = \mathcal{K}_r$.
Specifically, on the first step we obtain $\mathcal{K}_1$ by redefining $\mathcal{K}_0$ on $H_1$, such that for each $i\in[m]$
$$
    \left(\mathcal{K}_1(h_{1,i}), \mathcal{K}_1( g^{-1}\cdot h_{1,i})\right)=\mathbf{k}^{o_1(i), C - (m-i+1)\epsilon},
$$
where $\mathbf{k}^{o_1(i), C - (m-i+1)\epsilon}$ for $i\in [m]$ is defined using Lemma~\ref{example}, such that
\begin{align*}
    u_{2p} *_G\mathcal{K}_{1}(h_{1,i}) &=
   \widetilde{u}_{2p}(\mathbf{k}^{o_1(i), C - (m-i+1)\epsilon})< B, \text{ for all } p\in [m]\ and\ p\neq i \text{ and } \\
    u_{2i}*_G\mathcal{K}_{l}(h_{1,i}) &=
    \widetilde{u}_{2i}(\mathbf{k}^{o_1(i), C - (m-i+1)\epsilon}) = C - (m-i+1)\epsilon.
\end{align*}
Next, we obtain $\mathcal{K}_{l+1}$ by redefining the values of the the kernel $\mathcal{K}_{l}$ on $H_{l+1}$. Specifically, for $\widetilde{H}_l  \coloneqq \bigcup_{l' \leq l}H_l$, let
\begin{align}
    \label{def_ml}
    m_l &\coloneqq \min_{h\in \widetilde{H}_l ,\ i\in[m], (u_{2i}*_G\mathcal{K}_l)(h) > B} (u_{2i}*_G\mathcal{K}_l)(h),\\
    M_l &\coloneqq \max_{i, j\in[m]} |\hat{\nu}_{\mathcal{K}_l,i}(-m_l+\epsilon) - \hat{\nu}_{\mathcal{K}_l,j}(-m_l+\epsilon)|.
\end{align}
Lemma~\ref{example} guarantees for every $i\in [m]$ the existence of $\mathbf{k}^{o_{l+1}(i), m_l - (m-i+1)(M_l + \epsilon)}$, such that 
\begin{align}
    \widetilde{u}_{2p}(\mathbf{k}^{o_{l+1}(i), m_l - (m-i+1)(M_l + \epsilon)}) &< B, \text{ for all } p\in [m]\ and\ p\neq i \text{ and }\nonumber\\
    \widetilde{u}_{2i}(\mathbf{k}^{o_{l+1}(i), m_l - (m-i+1)(M_l + \epsilon)}) &= m_l - (m-i+1)(M_l + \epsilon),\label{eq:theValuesOfULieBetweenBAndC}
\end{align}
if the conditions of the lemma are satisfied, i.e., if $B < m_l - (m-i+1)(M_l + \epsilon) < C$ for all $l \leq |\mathcal{O}|$ and $i \in [m]$.
If the conditions hold, we define 
\begin{align}\label{eq:defOfKl}
\left(\mathcal{K}_{l+1}(h_{l+1,i}), \mathcal{K}_{l+1}(g^{-1}\cdot h_{l+1,i})\right) = 
\mathbf{k}^{o_{l+1}(i), m_l - (m-i+1)(M_l + \epsilon)}
\text{ for all }i \in [m].
\end{align}
We will prove that lemma~\ref{example} can be applied in Lemma~\ref{cond_lem} below. 

\begin{lemma}
\label{cond_lem}
For every $l \in [m]$ it holds that
\begin{align}\label{eq:ThisThingHoldsForAlll}
     B < m_{l} - m \cdot (M_{l} + \epsilon)\text{  and } M_{l} \leq \epsilon \cdot (m^{l} - 1).  
\end{align}
\end{lemma}
\begin{proof}
For every $h\in G$ and every $u\coloneqq a  \mathds{1}_e + b  \mathds{1}_g$, % where $e$ is the identity element of $G$, $g$ is the element in $G$, %
where $a,b \in \R$, it holds that 
\begin{align*}
    (u*_G\mathcal{K}_l)(h) &= a   \mathcal{K}_l(h) + b   \mathcal{K}_l(g^{
-1}   h) = \widetilde{u}(\mathcal{K}_l(h), \mathcal{K}_l(g^{-1}\cdot h)).   
\end{align*}
As a consequence of the construction, we have that for each $p\leq l$
\begin{align*}
  (u_{2o_p(j)}*_G\mathcal{K}_{l})(h_{p,i}) < B, \text{ for all } j \neq i \text{ and }
  u_{2o_p(i)}*_G\mathcal{K}_{l}(h_{p,i}) = m_{p-1} - (m-i+1)(M_{p-1} + \epsilon).
\end{align*}
In addition, as $\mathcal{K}_l(h) = 0$ for every $h\in \widetilde{H}_r\setminus\widetilde{H}_l $ it holds by construction that $g^{-1} \cdot h \notin \widetilde{H}_l$ and hence that 
\begin{align}
\label{equal_0}
    (u_{2j}*_G\mathcal{K}_{l})(h) = 0 < B,\text{ for every } h\in \widetilde{H}_r\setminus\widetilde{H}_l. 
\end{align}
Let us check \eqref{eq:ThisThingHoldsForAlll} for $l=1$.
By the construction of $\mathcal{K}_1$ for each $i\in[m]$ there is exactly one element $h\in \widetilde{H}_r$ such that $u_{2i}*_G\mathcal{K}_1(h) > m_1-\epsilon$.
Hence, $\hat{\nu}_{\mathcal{K}_1,{2o_1(i)}}(-m_1+\epsilon) = C - (m-1+i)\epsilon - m_1+\epsilon$.
Thus, $M_1 = (m-1)\epsilon$ and
\begin{align*}
m_1 - m \cdot (M_{1} + \epsilon) &= m_1 - m^2\epsilon = C - m\epsilon - m^2\epsilon
\\
&= C - \epsilon m(m+1) > C - \frac{(m+1)(C-B)}{(1 + m^{r})}\geq B,
\end{align*}
which implies the result for $l=1$.

Next, we check \eqref{eq:ThisThingHoldsForAlll} for $l=p+1$ if it holds for all $l\leq p$.
We first note that 
\begin{align}
\label{inequality}
m_p - (m-i+1)(M_p + \epsilon) &< m_p, \text{ for all }i \in [m], \\
m_p - (m-i+1)(M_p + \epsilon) &\geq m_p - m\cdot(M_p + \epsilon) > B. \nonumber
\end{align}
 
Consequently, by the construction of $\mathcal{K}_{p+1}$, for each $i \in [m]$ there is exactly one element $h \in \widetilde{H}_r$ such that $u_{2i}*_G\mathcal{K}_{p+1}(h) < m_p -\epsilon$, and $u_{2i}*_G\mathcal{K}_{p+1}(h) > B.$
Now, we are ready to estimate $M_{p+1}$:
\begin{align}
    &\hat{\nu}_{\mathcal{K}_{p+1},{2o_{p+1}(i)}}(-m_{p+1} + \epsilon)\nonumber \\
    &= 
    \left((u_{2o_{p+1}(i)}*_G\mathcal{K}_{p+1})(h_{p+1,i}) -m_{p+1} + \epsilon\right) \nonumber \\& \qquad +  \hat{\nu}_{\mathcal{K}_{p+1},{2o_{p+1}(i)}}(-m_{p} + \epsilon) + (m_p-m_{p+1})\sum_{g \in \widetilde{H}_r}\mathds{1}_{(u_{2o_{p+1}(i)}*_G\mathcal{K}_{p+1})(g) > m_p- \epsilon}\nonumber \\
    &=\left((u_{2o_{p+1}(i)}*_G\mathcal{K}_{p+1})(h_{p+1,i}) -m_{p+1} + \epsilon\right) + \hat{\nu}_{\mathcal{K}_{p+1},{2o_{p+1}(i)}}(-m_{p} + \epsilon) + (m_p-m_{p+1})p \nonumber \\
    &=\left(m_p - (m - i + 1)(M_p + \epsilon) -m_{p+1} + \epsilon \right) + \hat{\nu}_{\mathcal{K}_{p},{2o_{p+1}(i)}}(-m_{p} + \epsilon) + (m_p-m_{p+1})p.\label{gamma}
\end{align}
If we plug this equality to the definition of $M_{p+1}$, we receive:
\begin{align*}
M_{p+1} \leq |(m-1)(M_p + \epsilon)| + M_{p} = m   M_p + (m-1)   \epsilon \leq \left(m\cdot(m^p - 1) + (m-1)\right)   \epsilon = (m^{p+1} - 1)  \epsilon.
\end{align*}
It remains to prove that $m_{p+1} - m \cdot (M_{p+1} + \epsilon) > B$.
We have by construction that
\begin{align*}
    m_{p+1} &= m_p - m\cdot (M_p + \epsilon)\geq m_p - m  \epsilon    m^{p} = m_p - \epsilon    m^{p+1}\\
    &\geq m_{p-1} - \epsilon    m^{p} - \epsilon    m^{p+1}\geq\cdots\geq m_1 - \epsilon\cdot\left(m^2+m^3+\cdots+m^{p+1}\right)\\
    &= C - \epsilon   m\cdot\left(1+m+\cdots+m^{p}\right).
\end{align*}
Invoking the definition of $\epsilon$ \eqref{def_epsilon}, yields
\begin{align*}
    m_{p+1}&\geq C -  m\cdot\left(1+m+\cdots+m^{p}\right) \frac{(C-B)(m-1)}{2m(m - 1 + m^{r+1} - 1)} \\
    &=C - \frac{(C-B)(1+m+\cdots+m^{p})}{2(1+ 1+m+\cdots+m^{r})}\geq C - \frac12(C-B) = \frac12(C+B) > B.
\end{align*}
\end{proof}
Lemma \ref{cond_lem} shows that in every step $m_l - (m-i+1)(M_l + \epsilon)$ lies between $B$ and $C$. Therefore, by Lemma~\ref{example}, we find $\mathbf{k}^{o_{l+1}(i), m_l - (m-i+1)(M_l + \epsilon)}$ such that \eqref{eq:theValuesOfULieBetweenBAndC} holds. 

It remains to prove that the order $\hat{\sigma}(\mathcal{K}, -m_l+\epsilon)$ associated to $u_2, u_4, \dots u_{2m}$ is equal to $o_l$ for every $l = 1,2,\dots,r$. 
By construction, for each $h\in \widetilde{H}_r\setminus \widetilde{H}_{l}$ and each $j\in[m]$ it holds that  $(u_{2j}*_G\mathcal{K}_{r})(h) < m_l$.
Thus, we conclude that 
\begin{align*}
    (u_{2j}*_G\mathcal{K}_{r})(h) < m_l,\text{ for every } j\in[m]\text{ and every } h\in \widetilde{H}_r\setminus \widetilde{H}_{l} .   
\end{align*}
Therefore, $\hat{\nu}_{\mathcal{K},{i}}(c) = \hat{\nu}_{\mathcal{K}_l,{i}}(c)$ for every $c\leq-m_l+\epsilon$ and every $i\in [m]$.
For each $j>i$, after invoking \eqref{gamma}, we receive:
\begin{align}
    &\hat{\nu}_{\mathcal{K}_r,{o_l(i)}}(-m_{l}+\epsilon) - \hat{\nu}_{\mathcal{K}_r,{o_l(j)}}(-m_{l}+\epsilon)\nonumber\\
    &= \hat{\nu}_{\mathcal{K}_l,{o_l(i)}}(-m_{l}+\epsilon) - \hat{\nu}_{\mathcal{K}_l,{o_l(j)}}(-m_{l}+\epsilon)\nonumber\\
    &=\left(m_{l-1} - (m - i + 1)(M_{l-1} + \epsilon) + \hat{\nu}_{\mathcal{K}_{l}, {o_{l}(i)}}(-m_{l-1} + \epsilon)\right)\nonumber\\
    &\qquad -\left(m_{l-1} - (m - j + 1)(M_{l-1} + \epsilon) + \hat{\nu}_{\mathcal{K}_{l},{o_{l}(j)}}(-m_{l-1} + \epsilon)\right)\nonumber\\
    &=
    (j-i)(M_{l-1} + \epsilon) + \left(\hat{\nu}_{\mathcal{K}_{l},{o_{l}(i)}}(-m_{l-1} + \epsilon)-\hat{\nu}_{\mathcal{K}_{l},{o_{l}(j)}}(-m_{l-1} + \epsilon)\right)\nonumber\\
    &\geq
    M_{l-1} + \epsilon - M_{l-1} = \epsilon > 0.\label{eq:thevsarealldifferent}
\end{align}
Thus, $\hat{\nu}_{\mathcal{K}_r,{o_l(i)}}(-m_{l}+\epsilon) < \hat{\nu}_{\mathcal{K}_r,{o_l(j)}}(-m_{l}+\epsilon)$ for all $i<j$. Noting with Lemma~\ref{cond_lem} that $-m_{l}+\epsilon < -B$ yields the result.
\end{proof}

\begin{remark}\label{properties_of_kernel}
We list some properties of the kernel $\mathcal{K}$, and the functions $(\hat{\nu}_{\mathcal{K}, i})_{i\in[m]}$ defined in Lemma~\ref{lem:prepre_Constr_permutations}.
\begin{enumerate}
    \item Let for $B,C,m,r$ as in Lemma~\ref{lem:prepre_Constr_permutations}, $\epsilon$ be defined according to \eqref{def_epsilon}, and let for $l\in[r]$ $c_l\coloneqq m_l - 0.5\epsilon$, where $m_l$ is defined in \eqref{def_ml}. 
    Then, by construction of $\mathcal{K}$ for all $l\in [r], j\in[m]$ there does not exist $g\in \widetilde{H}_r$ such that 
    $$
    u_{2j}*_{G}\mathcal{K}(g) \in (c_l - 0.5\epsilon, c_l  + 0.5\epsilon ) = ( m_l -\epsilon, m_l).
    $$

    This property holds because, in the construction in \eqref{eq:theValuesOfULieBetweenBAndC}, we chose $\mathcal{K}$ such that all $u_{2j}*_{G}\mathcal{K}(g)$ for $j\in[m], g\in  \widetilde{H}_r$ lie between $m_l - (M_l + \epsilon)< m_l -  \epsilon$ and $m_{l+1}$ for some $l \in [r]$.

    Thus the order $\hat{\sigma}(\mathcal{K}, -c_l)$ is equal to the order $\hat{\sigma}(\mathcal{K}, -c_l+ 0.5\epsilon)$, which is by construction equal to $o_l$ for each $l\in[r]$. 
       
    \item It holds that $|\hat{\nu}_{\mathcal{K}, i}(c_l) - \hat{\nu}_{\mathcal{K}, j}(c_l)| > 0$ for all $l\in[r],\ i\in[m],\ j\in[m]$ with $i \neq j$, due to \eqref{eq:thevsarealldifferent}. 
\end{enumerate}    
\end{remark}

Next, we use Lemma~\ref{lem:prepre_Constr_permutations} to construct an expressive kernel.

\begin{lemma}\label{lem:Constr_permutations}
Let $C > B >0$ and $r, m \in \N$. Let $\mathcal{O} = \{o_1,\dots, o_r\} \subset \{ \sigma : [m] \to [m]\}$ be a set of orders of $[m]$. Let $G$ be a finite group with $|G|\geq 2 r m$ containing an element of order two, and let $\mu$ be the counting measure on $G$.
We define $u_0,\dots, u_{2m-1}$ as in Lemma~\ref{example}. 

Then, there exists a bounded kernel $\mathcal{K}:G\rightarrow\mathbb{R}$ such that $\mathcal{O}\subset \mathcal{O}(\mathcal{K})$, where $\mathcal{O}(\mathcal{K})$ is the set of all orders associated to $u_2, \dots, u_{2m}$ as in Definition \ref{def:setOfOrders}.
\end{lemma}

\begin{proof}
We start the proof by showing that for $l \in [r]$ the subsets $H_l \subset G$, described in Lemma~\ref{lem:prepre_Constr_permutations} exist.
Then, we will prove that for the kernel $\mathcal{K}$, defined in Lemma~\ref{lem:prepre_Constr_permutations}, the set of all orders associated to $u_2, \dots, u_{2m}$ contains $\mathcal{O}$.

To see that sets $H_l$ as above exist, we observe that we can build them sequentially: For $l \in [r]$, $p \in [m]$, because of the size of $G$ there exists $g' \in G$ such that 
$g' \not \in \left((g^{-1}\cdot H_{l'})\cup H_{l'}\right)$ for $l' < l$ and $g' \neq h_{l, q}$ and $g' \neq g^{-1} \cdot h_{l, q}$ for all $q < p$. Then, we set 
$h_{l, p} = g'$. 
We note that neither $g^{-1}\cdot h_{l, p} \in H_{l'}$ for $l' < l$ nor $g^{-1}\cdot h_{l, p} = h_{l, q}$ for some $q < p$. Indeed, if the opposite were the case then $g^{-1} \cdot h_{l, p}$ would either be an element of $H_{l'}$ or $g^{-1}\cdot h_{l, p} = h_{l, q}$ for some $q < p$. However, since $g$ is an element of order two it holds that $g^{-1}\cdot h_{l, p} = g\cdot h_{l, p}$ would be an element of  $H_{l'}$ or $g\cdot h_{l, p} = h_{l, q}$ for some $q < p$. This produces a contradiction to the choice of $h_{l, p}$. 

Having constructed subsets $H_l$, $l \in [r]$ such that the assumptions of Lemma~\ref{lem:prepre_Constr_permutations} are satisfied, we define the kernel $\mathcal{K}$ accordingly.
We now calculate the values of $(u_{2j}*_G\mathcal{K})(h)$ for $h\notin \widetilde{H}_r$ and $j=0,\dots, m-1$:
\begin{enumerate}
    \item for $h\in g^{-1}\cdot \widetilde{H}_r$ it holds that $(u_{2j}*_G\mathcal{K})(g^{-1}\cdot h) = (u_{2j+1}*_G\mathcal{K})(h) < B$ by the construction in Lemma~\ref{example}.
    \item for $h\in G\setminus((g^{-1}\cdot \widetilde{H}_r)\cup \widetilde{H}_r)$ it holds that $\mathcal{K}(h) = \mathcal{K}(g^{-1}\cdot h) = 0$, and hence 
    \begin{align}
        (u_{2j}*_G\mathcal{K})(h) = 0 < B,\text{ for every } h\notin \widetilde{H}_l . 
    \end{align}
\end{enumerate}

To summarize, it holds that $(u_{2j}*_G\mathcal{K})(h) < B$ for each $h\notin \widetilde{H}_r$ and $j\in [m]$.
Hence, for $c < -B$
$$
 \sum_{g \in G} \relu((u_{2k}*_G\mathcal{K})(g) + c) =
 \sum_{g \in \widetilde{H}_r} \relu((u_{2k}*_G\mathcal{K})(g) + c) = \hat{\nu}_{\mathcal{K},k}(c),
$$
for $k\in [m]$. 
Thus for $c < -B$ the order $\sigma(\mathcal{K}, c)$ associated to $u_2, u_4, \dots u_{2m}$ as defined in Definition~\ref{def:setOfOrders} is equal to the order $\hat{\sigma}(\mathcal{K},c)$, which is equal to $o_l$ for every $l = 1,2,\dots,r$.  
We conclude with Lemma~\ref{lem:prepre_Constr_permutations} that $\mathcal{O} \subset\widehat{\mathcal{O}}_B(\mathcal{K}) \subset \mathcal{O}(\mathcal{K})$. This completes the proof.
\end{proof}

\begin{comment}
\begin{remark}
For infinite groups, a version of Lemma~\ref{lem:Constr_permutations} holds, if all integrals are taken with respect to a specific finitely supported measure. 

Under the assumptions of Lemma~\ref{lem:Constr_permutations}, but with an infinite group $G$, we choose $\widetilde{H}_r$ in Lemma~\ref{lem:Constr_permutations}. Then, we define a measure $\mu$ on $G$, such that $\mu(h) = 1,$ for every $h\in H$, and $\mu(h) = 0$ for every $h\in G\setminus H$, where
$$
    H\coloneqq \{e, g\}\bigcup \left( (g^{-1}\cdot \widetilde{H}_r)\cup \widetilde{H}_r\right).
$$

Then, there exists a bounded kernel $\mathcal{K}:G\rightarrow\mathbb{R}$ such that $\mathcal{O}\subset \mathcal{O}(\mathcal{K})$, where $\mathcal{O}(\mathcal{K})$ is the set of all orders associated to $u_2, \dots, u_{2m}$ as in Definition \ref{def:setOfOrders} and all integrals are taken with respect to $\mu$.

This generalization holds since ...
\end{remark}
\end{comment}

In Lemma~\ref{lem:Constr_permutations}, we assumed that the group contains an element of order two.
However, this assumption is not necessarily satisfied when the cardinality of $G$ is odd. The following lemma drops the assumption at the cost of requiring a larger group.

\begin{lemma}
\label{lem:Constr_permutations_general}
Let $C > B >0$ and $r, m \in \N$. Let $\mathcal{O} = \{o_1,\dots, o_r\}\subset \{ \sigma : [m] \to [m]\}$ be a set of orders of $[m]$. 
Let $G$ be a finite group with $|G|\geq 9 r m$, which contains an element $g$ such that $g,\ g^{2} \neq e$, and let $\mu$ be the counting measure on $G$.
We define $u_0,\dots, u_{2m}$ as in Lemma~\ref{example}. 

Then, there exists a bounded kernel $\mathcal{K}:G\rightarrow\mathbb{R}$ such that $\mathcal{O}\subset \mathcal{O}(\mathcal{K})$, where $\mathcal{O}(\mathcal{K})$ is the set of all orders associated to $u_2, \dots, u_{2m}$ as in Definition \ref{def:setOfOrders}.
\end{lemma}

\begin{proof}
The proof consists of the following steps: 
\begin{enumerate}
    \item we first show how to construct sets $H_l \subset G ,\ l=1,\dots,r$ such that the assumptions of Lemma~\ref{lem:prepre_Constr_permutations} are satisfied,
    \item then define the kernel $\widetilde{\mathcal{K}}$ using Lemma~\ref{lem:prepre_Constr_permutations},
    \item and finally we construct $\mathcal{K}$ by modifying $\widetilde{\mathcal{K}}$ and prove that $\mathcal{O}\subset \mathcal{O}(\mathcal{K})$.
\end{enumerate}
For $l \in [r]$, we choose subsets $H_l \subset G$, such that
\begin{enumerate}%\label{properties_H_general}
    \item $|H_l| = m$,
    \item for every $h_{i,p}\in H_i$ and $h_{j,t}\in H_j$,
    $i,j\in[r],\  p,t\in[m]$, and $i\neq j$
    \begin{align*}
      \{g^{-2}\cdot h_{i,p},\ g^{-1}\cdot h_{i,p},\  h_{i,p},\ g\cdot h_{i,p},\ g^2\cdot h_{i,p}\}\cap\{g^{-2}\cdot h_{j,t},\ g^{-1}\cdot h_{j,t},\ h_{j,t},\ g\cdot h_{j,t},\ g^2\cdot h_{j,t}\}=\emptyset, 
    \end{align*}
    
    \item for every $h_{i,p}\in H_i$, $h_{i,t}\in H_i$
    $i\in[r],\  p,t\in[m]$ and $t\neq p$
    \begin{align*}
      \{g^{-2}\cdot h_{i,p},\ g^{-1}\cdot h_{i,p},\  h_{i,p},\ g\cdot h_{i,p},\ g^2\cdot h_{i,p}\}\cap\{g^{-2}\cdot h_{i,t},\ g^{-1}\cdot h_{i,t},\ h_{i,t},\ g\cdot h_{i,t},\ g^2\cdot h_{i,t}\}=\emptyset.  
    \end{align*}
\end{enumerate}
To see that $H_l$ as above exist we observe that we can build them sequentially: For $l \in [r]$, $p = 1,\dots,m$, because of the size of $G$ there exists $g' \in G$ such that 
$$
g' \not \in \bigcup_{i=1}^{l-1} \left( g^{-4} \cdot H_i \cup g^{-3} \cdot H_i \cup \cdots \cup \ g^{4} \cdot H_i\right)\cup \bigcup_{r=1}^{p-1} \{g^{-4}\cdot h_{l,r},\ g^{-3}\cdot h_{l,r},\dots,\ g^{3}\cdot h_{l,r},\ g^{4}\cdot h_{l,r}\}.
$$
Then, we set $h_{l, p} = g'$.
 
Let $\widetilde{\mathcal{K}}$ be the kernel defined in Lemma~\ref{lem:prepre_Constr_permutations}. As $\widetilde{\mathcal{K}}$ has non zero values only on the elements from $\widetilde{H}_r\cup (g^{-1}\cdot \widetilde{H}_r)$, we can define 
\begin{align*}
    K&\coloneqq\max_{g\in G}{|\widetilde{\mathcal{K}}(g)|},\\
    s&\coloneqq\min_{i\in[m], j\in[2]}a_{2i,j},\\
    S&\coloneqq\max_{i\in[m], j\in[2]}{a_{2i,j}},
\end{align*}
where by definition of $u_i,\ u_i\coloneqq a_{i,1}   \mathds{1}_{e} + a_{i, 2}   \mathds{1}_{g}$ for $i\in[2m]$.
Note that by construction, $a_{2i,1} > 0$ and $a_{2i,2} > 0$ holds for all $i\in [m]$, hence $s$ is positive. 

Now we define the kernel $\mathcal{K}$:
\begin{enumerate}
    \item The kernel $\mathcal{K}$ is equal to zero on $G\setminus\bigcup_{i=1,p=1}^{i=r,p=m} \{g^{-2}\cdot h_{i,p},\ g^{-1}\cdot h_{i,p},\  h_{i,p},\ g\cdot h_{i,p}\}$.
    \item For every $l\in[r]$ and $p\in [m]$, we define $\mathcal{K}(h_{l,p}) = \widetilde{\mathcal{K}}(h_{l,p})$, $\mathcal{K}(g^{-1}\cdot h_{l,p}) = \widetilde{\mathcal{K}}(g^{-1}\cdot h_{l,p})$.
    \item For every $l\in[r]$ and $p\in [m]$, we define $\mathcal{K}(g^{-2}\cdot h_{l,p}) = \mathcal{K}(g\cdot h_{l,p}) = -K\cdot\frac{S}{s}$.
\end{enumerate}

As a consequence of the construction, we have for all $i\in[r],\ p \in [m], j\in[2m]$ that 
\begin{align}\label{eq:asdfionaosid}
    (u_{j}*_G\mathcal{K})(h_{l,p}) = (u_{j}*_G\widetilde{\mathcal{K}})(h_{l,p}).
\end{align}
Moreover, for all $i\in[r],\ p, j\in[m]$ and $h\in \{g^{-2}\cdot h_{i,p},\ g^{-1}\cdot h_{i,p},\ g\cdot h_{i,p},\ g^2\cdot h_{i,p}\}$
\begin{align*}
    (u_{2j}*_G\mathcal{K})(h) = a_{2j,1}   \mathcal{K}(h) + a_{2j,2}   \mathcal{K}(g^{-1}\cdot h) \leq
    -s   K  \frac{S}{s} + K   S = 0,
\end{align*}
as $a_{2j,1}, a_{2j,2} \geq s$, and $h$ or $g^{-1}\cdot h$ is in $\{g^{-2}\cdot h_{i,p}, g\cdot h_{i,p}\}$, where we defined the kernel to be equal to $-K  \frac{S}{s}$.

On all elements $h$ from $G\setminus\bigcup_{i=1,p=1}^{i=r,p=m} \{g^{-2}\cdot h_{i,p},\ g^{-1}\cdot h_{i,p},\  h_{i,p},\ g\cdot h_{i,p},\ g^{2}\cdot h_{i,p}\}$ it holds that for each $j\in[m]$
\begin{align*}
   (u_{2j}*_G\mathcal{K})(h) = a_{2j,1}   \mathcal{K}(h) + a_{2j,2}   \mathcal{K}(g^{-1}\cdot h) = 0,
\end{align*}
as both $\mathcal{K}(g^{-1}\cdot h)$ and $\mathcal{K}(h)$ are equal to 0.

To summarize, the values of $(u_{2j}*_G\mathcal{K})(h) < B$ for each $h\notin \widetilde{H}_r$ and all $j\in[m]$.
Hence, for $c\leq -B$, 
$$
 \sum_{g \in G} \relu((u_{2k}*_G\mathcal{K})(g) + c) =
 \sum_{g \in \widetilde{H}_r} \relu((u_{2k}*_G\mathcal{K})(g) + c) = \hat{\nu}_{\mathcal{K},k}(c) = \hat{\nu}_{\widetilde{\mathcal{K}},k}(c),
$$
for $k\in [m]$, where the last equality follows by \eqref{eq:asdfionaosid}.

Thus, for $c < -B$, the order $\sigma(\mathcal{K}, c)$ associated to $u_2, u_4, \dots u_{2m}$ as defined in Definition~\ref{def:setOfOrders} is equal to the order $\hat{\sigma}(\widetilde{\mathcal{K}}, c)$. We conclude with Lemma~\ref{lem:prepre_Constr_permutations} that $\mathcal{O} \subset\widehat{\mathcal{O}}_B(\widetilde{\mathcal{K}}) \subset \mathcal{O}(\mathcal{K})$. This completes the proof.
\end{proof}

\begin{remark}
\label{remark:general_construction}
The kernel $\mathcal{K}$ of Lemma \ref{lem:Constr_permutations_general} satisfies the following properties:

\begin{itemize}
    \item Let for $B,C,m,r$ as in Lemma~\ref{lem:Constr_permutations_general}, $\epsilon$ be defined according to \eqref{def_epsilon}, and let for $i\in[r]$ $c_i\coloneqq m_i - 0.5\epsilon$, where $m_i$ is defined in \eqref{def_ml}. 

    Since the kernel $\mathcal{K}$ agrees with the kernel of Lemma \ref{lem:prepre_Constr_permutations} on all entries where it takes a positive value by \eqref{eq:asdfionaosid}, the conclusion of Remark \ref{properties_of_kernel} holds for $\mathcal{K}$ as well. However, it now holds for all elements of the group. 

    Concretely, for all $l\in [r], j\in[m]$ there does not exist $g\in G$ such that $c_l  + 0.5\epsilon > u_{2j}*_{G}\mathcal{K}(g) > c_l - 0.5\epsilon$.
    
    As a result, the order $\hat{\sigma}(\mathcal{K}, -c_l)$ is equal to the order $\hat{\sigma}(\mathcal{K}, -c_l+ 0.5\epsilon)$, which is by construction equal to $o_l$ for each $l\in[r]$. 
   
    \item It holds that $|\hat{\nu}_{\mathcal{K}, i}(c_l) - \hat{\nu}_{\mathcal{K}, j}(c_l)| > 0$ for all $l\in[r],\ i\in[m],\ j\in[m]$ with $i \neq j$, due to \eqref{eq:thevsarealldifferent}. 
    \end{itemize}   

Moreover, for an infinite group $G$, we can choose $(H_l)_{l = 1}^r \subset G$ like in the proof of Lemma~\ref{lem:Constr_permutations_general} such that  $\widetilde{H}_r = \bigcup_{l \leq r}H_l\subset G$. 
Moreover, for a measure $\mu$ such that $\mu(h) = 1$ for all $h\in H$, and $\mu(h) = 0$ for every $h\in G\setminus H$ where 
$$
    H\coloneqq \{e, g\}\bigcup \left( (g^{-2}\cdot \widetilde{H}_r)\cup (g^{-1}\cdot \widetilde{H}_r)\cup  \widetilde{H}_r \cup (g\cdot \widetilde{H}_r) \right),
$$
the statement of Lemma~\ref{lem:Constr_permutations_general} holds.
Indeed, in this case, the proof can be carried out mutatis mutandis. 
\end{remark}

In the Lemma~\ref{lem:Constr_permutations_general}, we assumed that the group contains a finite number of elements. The following theorem treats the case of infinite groups.

\begin{theorem}
\label{thm:infiniteGroupsHaveInfiniteVCDim}
Let $m \in \N$, let $G$ be an infinite compact group acting on $\mathcal{X}$, and let $\mu$ be the Haar measure on $G$.
Then, for every $n \in \N$ there is a bounded kernel $\mathcal{K} \colon G \to \R$ such that $\vcdim(\mathcal{H}(\mathcal{K})) \geq n$. 
\end{theorem}
\begin{proof}
We fix $m\in\mathbb{N}$, $g\in G,\ g\neq e$ and $u_0,\dots,u_{2m}$ as in Lemma~\ref{example}, and let $\mathcal{O} \coloneqq \{o_1,\dots,o_r\}$ be a complete set of orders. 

Next, we show how to construct a kernel $\mathcal{K}$, such that $\mathcal{O}\subset \mathcal{O}(\mathcal{K})$, where $\mathcal{O}(\mathcal{K})$ is the set of all orders associated to $u_2, \dots, u_{2m}$ as in Definition \ref{def:setOfOrders}. 

As the first step of the proof, we introduce an auxiliary measure $\mu_1$.
This will be the counting measure on a finite subset of $G$.
To clarify which measure is used in the convolutions in the sequel, we will use in this proof the notation $f*_{\mu_1, G}\mathcal{K}$ for the convolution of a function $f \colon G\to \R$ with the kernel $\mathcal{K} \colon G\to \R$ using the measure $\mu_1$.

As we already mentioned in Remark~\ref{remark:general_construction}, since $G$ is infinite, we can choose finite sets $H_l \subset G ,\ l=1,\dots,r$, define 
\begin{align*}
  H\coloneqq \{e, g\}\bigcup\left(\bigcup_{i=1,p=1}^{i=r,p=m} \{g^{-2}\cdot h_{i,p},\ g^{-1}\cdot h_{i,p},\  h_{i,p},\ g\cdot h_{i,p}\}\right),\text{ where } h_{i,p}\in H_i,  
\end{align*}
and choose $\mu_1$ to be the counting measure on $H$. Then, Lemma~\ref{lem:Constr_permutations_general} can be applied with $\mu_1$ to yield the filter $\widehat{\mathcal{K}}$, such that $\mathcal{O}\subset \mathcal{O}(\widehat{\mathcal{K}})$, where all convolutional operations are integration by measure $\mu_1$. 

However, since $\mu \neq \mu_1$, it is not necessarily the case that $\mathcal{H}(\widehat{\mathcal{K}})$ shatters $u_2, \dots, u_{2m}$. To correct this, we choose an open neighborhood of the identity $U$ such that $h\cdot U\cap h'\cdot U = \emptyset$ for all $h\neq h',\ h, h'\in H$. The existence of $U$ follows from the Hausdorff property of $G$. Indeed, we can construct disjoint open sets $(U_{h})_{h \in H}$ such that $h \in U_h$ for all $h \in H$. Then, we set
\begin{align*}
    U\coloneqq \bigcap_{h\in H} h^{-1}\cdot (h\cdot U_e \cap U_h).
\end{align*}
Since we only take finitely many intersections in the construction of $U$, it is clear that $U$ is open. Moreover, we directly see that $e \in U$. Assuming there exists $h, h' \in H$ such that  $h\cdot U\cap h'\cdot U \neq \emptyset$ implies by construction that
\begin{align}\label{eq:anonemptyintersection}
    h\cdot U_e \cap U_h \cap h' U_e \cap U_{h'} \neq \emptyset.
\end{align}
Clearly, \eqref{eq:anonemptyintersection} can only hold if $h=h'$. This shows that $U$ as desired exists.

We define a kernel $\mathcal{K}$ as a modification of the kernel $\widehat{\mathcal{K}}$: 
\begin{align}
 \mathcal{K}\coloneqq \sum_{h\in H} \widehat{\mathcal{K}}(h)\cdot\mathds{1}_{h\cdot U}.
\end{align}

To continue further, we introduce additional notation.
We define $a_{i}$ and $b_{i}$ as the coefficients $a_{i}  \mathds{1}_e + b_{i}  \mathds{1}_g \coloneqq u_{2i}$ for $i\in[2m]$. Moreover, we define for $h' \in G$
\begin{align*}
    \mathcal{K}_{a_{i}, b_{i}}(h') \coloneqq \sum_{h\in H} (a_{i} \widehat{\mathcal{K}}(h) +  b_{i} \widehat{\mathcal{K}}(g^{-1} \cdot h) )\cdot\mathds{1}_{h\cdot U}(h').
\end{align*}
We set
\begin{align*}
    \tilde{\nu}_{\mathcal{K}, i}:\mathbb{R}\rightarrow\mathbb{R},\ \tilde{\nu}_{\mathcal{K}, i}(c)\coloneqq  \int_{G}\relu(\mathcal{K}_{a_{2i}, b_{2i}}(h) - c)d\mu(h).
\end{align*}
For $(c_j)_{j=1}^r$ defined as in Remark~\ref{remark:general_construction}, we will show now, that for $i\in[m]$ the newly defined $\tilde{\nu}_{\mathcal{K}, i}(c_j)$ are equal to $\mu(U)\cdot \nu_{\widehat{\mathcal{K}}, i}(c_j)$.
Indeed,
\begin{align*}
    \tilde{\nu}_{\mathcal{K}, i}(c_j) &= \int_{G}\relu(\mathcal{K}_{a_{2i}, b_{2i}}(h_1) -c_j)d\mu(h_1)\\
    &=\int_{G}\sum_{h\in H}\left(\relu\left( \left(a_{2i}\widehat{\mathcal{K}}(h)\mathds{1}_{h\cdot U}(h_1) + b_{2i} \widehat{\mathcal{K}}(g^{-1} \cdot h) \cdot\mathds{1}_{h\cdot U}(h_1) \right) -c_j\right)\right)d\mu(h_1)\\
    &=\sum_{h\in H}\left(\int_{G}\relu\left(\left(a_{2i}\widehat{\mathcal{K}}(h) + b_{2i} \widehat{\mathcal{K}}(g^{-1}\cdot h)\right)\cdot\mathds{1}_{h\cdot U}(h_1) -c_j\right)d\mu(h_1)\right)\\
    &= \mu(U)\int_{G}\relu (u_{2i}*_{\mu_1,G}\widehat{\mathcal{K}} - c_j)(h)d\mu_1(h)= \mu(U)\cdot \nu_{\widehat{\mathcal{K}}, i}(c_j),
\end{align*}
where $\nu_{\widehat{\mathcal{K}}, i}$ is as in Definition \ref{def:setOfOrders}. The second equality holds since $a_{2i}\widehat{\mathcal{K}}(h) + b_{2i} \widehat{\mathcal{K}}(g^{-1}\cdot h)$ is greater than $c_j$ only for $h\in \widetilde{H}_r = \bigcup_{l\in[r]} H_l$;
the fourth equality holds since for a Haar measure $\mu(U) = \mu(h\cdot U)$.

From the proven equality, we can conclude that the order of $(\tilde{\nu}_{\mathcal{K}, i}(c_j))_{i\in[m]}$ is $o_j$ for every $j\in[r]$.
Moreover, let
\begin{align*}
    \tilde{m}&\coloneqq\min_{l\in[r],\ i, j\in[m], i\neq j} |\tilde{\nu}_{\mathcal{K}, i}(c_l) - \tilde{\nu}_{\mathcal{K}, j}(c_l)| \\
    &= \min_{l\in[r],\ i, j\in[m], i\neq j} |\nu_{\widehat{\mathcal{K}}, i}(c_l) - \nu_{\widehat{\mathcal{K}}, j}(c_l)|\cdot \mu(U) \\
    &= 
    \min_{l\in[r],\ i, j\in[m], i\neq j} |\hat{\nu}_{\widehat{\mathcal{K}}, i}(c_l) - \hat{\nu}_{\widehat{\mathcal{K}}, j}(c_l)|\cdot \mu(U)
    ,
\end{align*}
where $\hat{\nu}$ is defined in Lemma~\ref{lem:prepre_Constr_permutations}. 

It holds that $\tilde{m} > 0$, since by Remark~\ref{remark:general_construction} 
$$
    \min_{l\in[r],\ i, j\in[m], i\neq j}|\hat{\nu}_{\widehat{\mathcal{K}}, i}(c_l) - \hat{\nu}_{\widehat{\mathcal{K}}, j}(c_l)| > 0.
$$

We finish the proof, by modifying $u_i,\ i\in[2m]$ to yield $\widetilde{u}_i$, $i\in[2m]$, and showing that $\widetilde{u}_{2i}$, $i\in[m]$ are shattered by $\mathcal{H}(\mathcal{K})$.

We claim that for every $\delta >0$ there exists open sets $\widetilde{U}_{\delta}$ such that for all $i \in [2m]$
\begin{align} \label{eq:claimUTilde}
    \norm{\mu(\widetilde{U}_{\delta})^{-1}(a_{i}\mathds{1}_{\widetilde{U}_{\delta}^{-1}} + b_{i}\mathds{1}_{(g \cdot  \widetilde{U}_{\delta})^{-1}})*_{G}\mathcal{K} - \mathcal{K}_{a_{i}, b_{i}}}_{L^1} \leq \delta.
\end{align}
Before we prove \eqref{eq:claimUTilde}, we show how it yields the claim. We fix 
$\delta \coloneqq \tilde{m}/{4},
$
and denote the functions $\mu(\widetilde{U}_{\delta})^{-1}(a_{i}\mathds{1}_{\widetilde{U}_{\delta}^{-1}} + b_{i}\mathds{1}_{(g\cdot  \widetilde{U}_{\delta})^{-1}})*_{G}\mathcal{K}$ as $\widetilde{\mathcal{K}}_{a_{i},b_{i}}$, and $\widetilde{u}_{i}\coloneqq \mu(\widetilde{U}_{\delta})^{-1}(a_{i}\mathds{1}_{\widetilde{U}_{\delta}^{-1}} + b_{i}\mathds{1}_{(g\cdot  \widetilde{U}_{\delta})^{-1}})$.

We prove that $\mathcal{H}(\mathcal{K})$ shatters $(\widetilde{u}_{2i})_{i\in[m]}$ by showing that for each $j\in[r]$ the order of $(\int_G \relu(\widetilde{\mathcal{K}}_{a_{2i},b_{2i}}(g) -c_j)d\mu(g))_{i\in[m]}$ is equal to $o_j$. 
Let $i, j\in[m]$, and $l\in[r]$, such that $\tilde{\nu}_{\mathcal{K}, i}(c_l) - \tilde{\nu}_{\mathcal{K}, j}(c_l) \geq \tilde{m} > 0$. 
Then, we can conclude that 
\begin{align*}
&\int_G \relu(\widetilde{\mathcal{K}}_{a_{2i},b_{2i}} - c_l)(h)d\mu(h) - \int_G \relu(\widetilde{\mathcal{K}}_{a_{2j},b_{2j}} - c_l)(h)d\mu(h) \\
&=  \left(\int_G \relu(\mathcal{K}_{a_{2i},b_{2i}} - c_l)(h)d\mu(h) - \int_G \relu(\mathcal{K}_{a_{2j},b_{2j}} - c_l)(h)d\mu(h)\right)\\
& \qquad - \left(\int_G \relu(\widetilde{\mathcal{K}}_{a_{2j},b_{2j}} - c_l)(h)d\mu(h) - \int_G \relu(\mathcal{K}_{a_{2j},b_{2j}} - c_l)(h)d\mu(h)\right)\\
& \qquad \qquad + \left(\int_G \relu(\widetilde{\mathcal{K}}_{a_{2i},b_{2i}} - c_l)(h)d\mu(h) - \int_G \relu(\mathcal{K}_{a_{2i},b_{2i}} - c_l)(h)d\mu(h)\right)\\
&\geq  \tilde{\nu}_{\mathcal{K}, i}(c_l) - \tilde{\nu}_{\mathcal{K}, j}(c_l)  - \|\widetilde{\mathcal{K}}_{a_{2j},b_{2j}} - \mathcal{K}_{a_{2j},b_{2j}}\|_{L^1} - \|\widetilde{\mathcal{K}}_{a_{2i},b_{2i}} - \mathcal{K}_{a_{2i},b_{2i}}\|_{L^1}\\
&\geq \tilde{m} - \frac{\tilde{m}}{4} - \frac{\tilde{m}}{4} = \frac{\tilde{m}}{2} > 0,
\end{align*}
where we used the linearity and monotonicity of the integral and the 1-Lipschitz property of the ReLU in the first inequality.

Thus, for each $j\in[r]$ the order of $(\int_G \relu(\widetilde{\mathcal{K}}_{a_{2i},b_{2i}}(g) -c_j)d\mu(g))_{i\in[m]}$ is equal to the order of $(\tilde{\nu}_{\mathcal{K}, i}(c_j))_{i\in[m]}$, which is equal to $o_j$.

We complete the proof by showing \eqref{eq:claimUTilde}.
First note, that since $G$ is assumed to be first-countable, there exists a sequence of neighborhoods of $e$, denoted by $(N_i)_{i\in \N}$, such that 
$$
    N_1 \supset N_2 \supset N_3 \supset \dots
$$
and, for all neighborhoods $Z$ of $e$, there exists $k \in \N$ such that $N_k \subset Z$. As a consequence, we have that 
\begin{align}\label{eq:constructionOfFilter}
    \bigcap_{i=1}^\infty N_i = \{e\}.
\end{align}
Indeed, assuming that there exists $g \in G, g \neq e$ with 
$g \in N_i$ for all $i \in \N$ yields a contradiction by invoking the Hausdorff property of $G$. 
Concretely, by the Hausdorff property, we have that there exists $U$ with $e \in U$, $g \not \in U$. Hence, there exists $k$ such that $N_k \subset U$ which implies $g \not \in N_k$ and produces the contradiction. 

Since $\mu$ is finite, we conclude that $\mu(\{e\}) = 0$. Moreover, since $\mu$ is a Borel measure, we have by \eqref{eq:constructionOfFilter} that 
$\mu(N_k) \to 0$ for $k \to \infty$. 

In addition, it follows from the continuity of the multiplication that there exists $(Q_k)_{k\in \N}$ such that $Q_k$ is open, $Q_k \subset N_k$ and $Q_k\cdot Q_k \subset N_k$ for all $k \in \N$.

Next, observe that to show \eqref{eq:claimUTilde}, it follows by the triangle inequality, linearity of the convolution, and the translation invariance of the Haar measure, that it suffices to show that for every $\delta >0$ there exists $\widetilde{U}_{\delta}$ such that for all $|a| \leq A \in \R^+$
\begin{align}\label{eq:simplificationofConvolutionStatement}
\norm{a \mu(\widetilde{U}_{\delta})^{-1}\mathds{1}_{\widetilde{U}_{\delta}^{-1}}*_{G}\mathcal{K} - a\mathcal{K}}_{L^1} \leq \delta/2.
\end{align}

Since for finite Borel measures on a compact set $G$, the set of continuous functions is dense in $L^1(G)$ (see \cite[Theorem 7.9]{follandRealAnalysis}), we can replace (also using Young's convolution inequality) $\mathcal{K}$ by a continuous approximation $\mathcal{K}^{\rm cont} \colon G \to \R$ such that 
\begin{align}\label{eq:simplificationofConvolutionStatement2}
\norm{a\mu(\widetilde{U}_{\delta})^{-1}\mathds{1}_{\widetilde{U}_{\delta}^{-1}}*_{G}\mathcal{K} - a\mathcal{K}}_{L^1} \leq \norm{a \mu(\widetilde{U}_{\delta})^{-1}\mathds{1}_{\widetilde{U}_{\delta}^{-1}}*_{G}\mathcal{K}^{\rm cont} - a\mathcal{K}^{\rm cont}}_{L^1} + \delta/4.
\end{align}
The proof is completed by observing that
\begin{align}
    &\norm{a \mu(\widetilde{U}_{\delta})^{-1}\mathds{1}_{\widetilde{U}_{\delta}^{-1}}*_{G}\mathcal{K}^{\rm cont}- a \mathcal{K}^{\rm cont}}_{L^1}\nonumber\\
    &= |a|\int_{G} \left| \mu(\widetilde{U}_{\delta})^{-1} \int_{\widetilde{U}_{\delta}\cdot g} \mathcal{K}^{\rm cont}(h) - \mathcal{K}^{\rm cont}(g) d\mu(h)\right| d\mu(g).\label{eq:weFirstNeedToProveUniformContinuity}
\end{align}
Since $\mathcal{K}^{\rm cont}$ is continuous there exists for all $g \in G$ an open set set $Q_{k,g} \ni e$ such that $|\mathcal{K}^{\rm cont}(h) - \mathcal{K}^{\rm cont}(g)| \leq \delta/(8 A \mu(G))$ for all $h \in Q_{k,g} \cdot Q_{k,g} \cdot g$. 
Since $(Q_{k,g} \cdot g)_{g \in G}$ is an open cover of $G$, we can choose a finite subcover, $(Q_{k_i,g_i} \cdot {g_i})_{i=1}^K$ for $K \in \N$. 

We set 
$$
\widetilde{U}_{\delta} =  \bigcap_{i=1}^K Q_{k_i,g_i}
$$
and observe that $\widetilde{U}_{\delta}$ is an open neighborhood of $e$. Moreover, per construction $\widetilde{U}_{\delta} \cdot {g} \subset Q_{k_i,g_i} \cdot Q_{k_i,g_i} \cdot {g_i}$ for some $i \in [K]$. 
Therefore, 
\begin{align}\label{eq:ufhiasdg}
    |\mathcal{K}^{\rm cont}(h) - \mathcal{K}^{\rm cont}(g_i)| \leq \delta/(8 A \mu(G))
\end{align}
for all $h \in \widetilde{U}_{\delta} \cdot g$. Since \eqref{eq:ufhiasdg} holds in particular for $h = g$, we conclude by the triangle inequality that 
$|\mathcal{K}^{\rm cont}(h) - \mathcal{K}^{\rm cont}(g)| \leq \delta/(4 A \mu(G))$
for all $h \in \widetilde{U}_{\delta} \cdot g$.

We conclude with \eqref{eq:weFirstNeedToProveUniformContinuity} that
\begin{align*}
\norm{a \mu(\widetilde{U}_{\delta})^{-1}\mathds{1}_{\widetilde{U}_{\delta}^{-1}}*_{G}\mathcal{K}^{\rm cont}- a\mathcal{K}^{\rm cont}}_{L^1} \leq \delta/4,
\end{align*}
which yields \eqref{eq:simplificationofConvolutionStatement2}, and hence \eqref{eq:simplificationofConvolutionStatement} which implies \eqref{eq:claimUTilde} and completes the proof.
\end{proof}
\bibliographystyle{abbrv}
\bibliography{references}
\end{document}